\documentclass[mnsc]{informs3aa}
\usepackage[final]{showkeys}
\usepackage{subfigure}
\OneAndAHalfSpacedXI 

\usepackage{endnotes}

\usepackage{longtable}
\usepackage{tabularx}
\usepackage{booktabs}
\usepackage{etex}
\usepackage{multirow}
\usepackage{array}
\usepackage{bbm}

\usepackage[title]{appendix}

\usepackage[outdir=./]{epstopdf}

\newcommand{\vct}{\boldsymbol }

\newcommand{\kl}{\mathrm{KL}}

\newcommand{\ud}{\mathrm{d}}

\newcommand{\mA}{\mathcal A}

\newcommand{\mX}{\mathcal X}
\newcommand{\mU}{\mathcal U}
\newcommand{\mL}{\mathcal L}
\newcommand{\mF}{\mathcal F}
\newcommand{\mE}{\mathcal E}

\newcommand{\mW}{\mathcal W}

\newcommand{\trim}{\mathrm{Trim}}

\newcommand{\mZ}{\mathcal Z}
\newcommand{\diam}{\mathrm{diam}}

\renewcommand{\hat}{\widehat}
\renewcommand{\tilde}{\widetilde}

\usepackage{amsbsy, amsfonts, amsgen, amsmath, amsopn, amssymb, amstext,
	amsxtra, bezier, color, enumitem, graphicx, latexsym, verbatim,
	pictexwd, supertabular, url, dsfont,leftidx, mathrsfs, appendix, setspace}
\usepackage{algorithm}
\usepackage{algorithmicx}
\usepackage{algpseudocode}
\makeatletter
\def\BState{\State\hskip-\ALG@thistlm}
\makeatother

\usepackage{natbib}
\bibpunct[, ]{(}{)}{,}{a}{}{,}%
\def\bibfont{\small}%
\usepackage[colorlinks=true,bookmarks=false,urlcolor=blue, citecolor=blue,linkcolor=blue,bookmarksopen=false,draft=false]{hyperref}

\usepackage{xcolor}
\usepackage{epstopdf}

\renewcommand{\hat}{\widehat}
\renewcommand{\tilde}{\widetilde}
\renewcommand{\bar}{\overline}

\newcommand{\cov}{\mathrm{Cov}}

\definecolor{DSgray}{cmyk}{0,1,0,0}

\TheoremsNumberedThrough     
\ECRepeatTheorems

\EquationsNumberedThrough    

\MANUSCRIPTNO{} 

\begin{document}
	
	
	
	\RUNTITLE{Fair Contextual Pricing}
	
	\TITLE{Utility Fairness in Contextual Dynamic Pricing with Demand Learning}

	
	\ARTICLEAUTHORS{%
		\AUTHOR{Xi Chen\thanks{Author names listed in alphabetical order.}}
		\AFF{Leonard N.~Stern School of Business, New York University, New York, NY 10012, USA\\
		xc13@stern.nyu.edu}
		\AUTHOR{David Simchi-Levi}
		\AFF{Institute for Data, Systems and Society, Operations Research Center, Department of Civil and Environmental Engineering, Massachusetts Institute of Technology, Cambridge, MA 02139, USA\\
		dslevi@mit.edu}
		\AUTHOR{Yining Wang}
		\AFF{Naveen Jindal School of Management, University of Texas at Dallas, Richardson, TX 75080, USA\\
		yining.wang@utdallas.edu}
	} %

	\ABSTRACT{This paper introduces a novel contextual bandit algorithm for personalized pricing under utility fairness constraints in scenarios with uncertain demand, achieving an optimal regret upper bound. Our approach, which incorporates dynamic pricing and demand learning, addresses the critical challenge of fairness in pricing strategies. We first delve into the static full-information setting to formulate an optimal pricing policy as a constrained optimization problem. Here, we propose an approximation algorithm for efficiently and approximately computing the ideal policy.
	We also use mathematical analysis and computational studies to characterize the structures of optimal contextual pricing policies subject to fairness constraints,
	deriving simplified policies which lays the foundations of more in-depth research and extensions.
	 Further, we extend our study to dynamic pricing problems with demand learning, establishing a non-standard regret lower bound that highlights the complexity added by fairness constraints. Our research offers a comprehensive analysis of the cost of fairness and its impact on the balance between utility and revenue maximization. This work represents a step towards integrating ethical considerations into algorithmic efficiency in data-driven dynamic pricing.
	}
	
	\KEYWORDS{Contextual bandit, Demand learning, Dynamic pricing, Minimax regret, Utility fairness}
	
	
	\date{}
	
	\maketitle

	\parskip=6pt
	
	\section{Introduction}\label{sec:intro}

	
	Personalized dynamic pricing, a pivotal strategy in modern revenue management, leverages customer-specific data to tailor prices. This approach maximizes revenue by aligning prices with individual customer value perceptions and willingness to pay. The advent of big data analytics and machine learning has significantly advanced the capabilities of personalized dynamic pricing. Businesses can now analyze vast amounts of data in real-time to understand customer behaviors, preferences, and price sensitivities. This advancement has led to more sophisticated and granular pricing strategies, where prices can be tailored not only to segments of customers but to an individual consumer's contextual profile. In practical scenarios, e-commerce platforms typically implement personalized pricing strategies through the use of discounts or promotions, rather than directly varying the prices.
	
	However, the efficiency of dynamic pricing comes with complexities and challenges, particularly concerning fairness and ethical considerations. As businesses increasingly leverage customer data to optimize pricing strategies, there is a growing concern about the fairness of these practices, both from a customer perspective and within the regulatory framework. Regulators have intensified scrutiny over potential discrimination from pricing algorithms. For instance, regulatory bodies like the U.S. Federal Trade Commission (FTC) and the UK's Financial Conduct Authority (FCA) mandate that retailers disclose their pricing algorithms \citep{ftc2020,fca2018}. These regulators then meticulously examine the design, inputs, and outputs of such algorithms to verify compliance with applicable fairness regulations. Retailers, in turn, must proactively implement fairness-aware dynamic pricing strategies. This approach is crucial to mitigate customer dissatisfaction and reputational damage arising from discriminatory pricing practices, while concurrently adhering to regulatory standards.
	
	Indeed, fairness is not just a moral imperative but a business necessity.  Unfair pricing practices can lead to customer dissatisfaction, erosion of trust, and potential backlash, both from consumers and regulators. In the long term, these factors can adversely affect a company's reputation and profitability. Fair pricing policies must ensure that customers with similar characteristics and behaviors are offered similar prices. Therefore, in this paper, we consider a utility fairness (see Definition \ref{defn:fairness}) concept, which requires the implemented pricing policy to offer similar prices to customers whose utility values are similar, even though their feature vector could be very different. 
	Our concept of utility fairness is encapsulated by a Lipschitz continuity constraint, which measures the disparity between pricing policy and utility for a given pair of contextual vectors. The Lipschitz constant, denoted as $\delta_0$, is indicative of the fairness constraint's intensity. Specifically, a lower value of $\delta_0$ signifies a more stringent fairness requirement. The advantage of utility fairness is that it will discriminate against a specific feature for the ``absolute'' fairness. Instead, it treats the entire customer feature vector as a whole for the purpose of fairness.
	
	By integrates fairness into personalized/contextual pricing, the framework balances profitability with ethical pricing practice. We study both the static (a.k.a., full-information) and multi-period dynamic pricing problems.
	The proposed policies serve as a significant step in marrying ethical considerations with revenue objective.
	The structure of the paper and main contributions are summarized as follows.
	
	\begin{enumerate}
		\item We first study the static full-information setting, where the demand function is fully known. We formulate optimal pricing policy as a constrained optimization problem (see Eqs. \eqref{eq:inf-opt-ver1} and \eqref{eq:inf-opt-ver1}). To find a near-optimal optimal policy with general demand function, we introduce an approximation algorithm based on discretization of the utility space. 
		This methodology offers a pragmatic solution to approximate the ideal policy within a polynomial time frame (see Sec. \ref{sec:fptas}).
		
		\item In additional to the discretization and approximation algorithm, we further provide analytical results that characterize the structure of optimal fair pricing policies, shedding light on the impact of fairness constraints (see Sec.~\ref{sec:structure}). Under an assumption of the demand function (see Eq. \eqref{eq:defn-sigma-u}) that is satisfied by many popular models, we establish several key properties of the optimal contextual pricing policy as a function of the utility (see Theorem \ref{thm:optimal-pi}). 
		In particular, we demonstrate that the optimal contextual pricing policy is \emph{monotonically non-decreasing}, \emph{differential almost everywhere}, with
		the values of the pricing policy determined by fairness unconstrained prices when the differential of the underlying policy satisfies certain conditions.
		Furthermore, when the Lipschitz parameter $\delta_0$ is not too large, we show that the optimal policy possesses a \emph{linear structure} in the sense that 
		the policy's offered prices are linear in customers' baseline utility values.

		This  result yields several interesting insights. First, with a linear demand function, it fully defines the structure of the optimal contextual pricing policy under fairness constraints (see Remark \ref{rem:rho-linear}). Second, for a linear demand function, we derive the closed-form expression for the  \emph{cost of utility fairness}, defined as the ratio of optimal revenue under fairness-constrained conditions to that under no-constraint conditions (see Theorem \ref{thm:cost-uf}). 
		This explicitly expresses the intrinsic trade-off between fairness and profit.
		 For other types of demands, we provide numerical illustration of the \emph{cost of utility fairness} based on our discretization algorithm. This analysis helps in understanding the trade-offs involved in implementing fairness in pricing strategies.
			
		We also remark that the structural results established here are highly non-trivial
		because the formulation of the optimal pricing policy (subject to fairness constraints) is essentially an infinite-dimensional optimization problem.
		Our proof relies on an infinite sequence of finite-dimensional optimization problems
		to approximate the infinite-dimensional problem, together with Lagrangian multiplier analysis and continuity/compactness in function spaces.
		
		 
		\item We further study the fair contextual pricing with unknown demand function in a $T$-period learning-while-doing setting (see Sec.~\ref{sec:fullinfo})  We first explore the regret lower bound to enhance our comprehension of the fairness constraint's influence in dynamic contextual pricing (see Sec.~\ref{sec:lower}).
			
		A rather surprising result is the problem becomes intrinsically more difficult due to the fairness constraint. We establish the $\Omega(T^{2/3})$ regret lower bound, which is different from the standard $O(\sqrt{T})$ or $O(\log T)$ bounds in existing dynamic pricing or contextual bandit problems \citep{rusmevichientong2010linearly,li2011pricing}. To understand the rationale behind this phenomenon, it is insightful to explore the structure of the optimal pricing policy in complete information setting (see Theorem \ref{thm:optimal-pi}). The structure reveals that the optimal price for a user with a specific context $x$  is linearly in $x$, indicating a perfect co-linearity between context vectors and offered prices in an optimal strategy. This presents a dilemma: implementing a near optimal pricing policy results in high-co-linearity between the context vector $x$ and the price $p$, hindering actuate estimations of the parameters. Conversely, adopting pricing strategies to avoid such co-linearity leads to suboptimal outcomes, deviating from the structure delineated in Theorem \ref{thm:optimal-pi}. This challenge results in an unconventional regret lower bound of  $\Omega(T^{2/3})$,  which underscores the impact of the fairness constraint from an information-theoretical standpoint.

	\item We introduce a new contextual bandit algorithm designed for personalized pricing under utility fairness constraints and uncertain demand, featuring a rate-optimal regret upper bound of \(O(T^{2/3})\) (see Sec.~\ref{sec:algo}). This algorithm comprises two phases: The initial phase involves price experimentation, during which maximum likelihood estimates are formulated from the gathered data. The subsequent phase applies the upper-confidence bound (UCB) algorithm, where each arm represents a discretized price for the zero utility value,
	with prices for customers with other utility values set according to the optimal linear structure of fairness constrained policies (see Algorithm \ref{alg:bandit-uf} for more details).	
	 
	\end{enumerate}

	In summary, this paper provides a rigorous approach to integrate utility fairness into personalized contextual pricing. For both static and dynamic settings, we fully explore the cost of the fairness constraint in the optimal revenue and regret bound. We believe this work offers significant steps towards bridging the gap between algorithmic efficiency and ethical considerations in data-driven dynamic pricing.
	Our framework could potentially be applied and extended to other types of revenue management problems.

	\section{Related Works} 
	\label{sec:related}
	
	This section briefly reviews two pertinent research areas: dynamic pricing and some recent development on fairness 
	
	\medskip
	\noindent {\bf Dynamic Pricing.}
	With the growth of online retailing, dynamic pricing has emerged as a prominent research topic in recent years. For extensive surveys, see \citep{bitran2003overview,elmaghraby2003dynamic, den2015dynamic}. Our focus is on the single-product pricing problem. The foundational work of \citet{gallego1994optimal} established the basis of dynamic pricing. Traditional models often assumed known demand information, but modern industries like fast fashion face unpredictable demand, sparking research on dynamic pricing with demand learning (see, e.g., \cite{araman2009dynamic, besbes2009dynamic,farias2010dynamic,broder2012dynamic,harrison2012bayesian, broder2012dynamic, den2013simultaneously,keskin2014dynamic,wang2014close,lei2014near, keskin2014dynamic, chen2015real, Bastani:21:meta,  Wang:21:uncertainty, Miao:19}).

	The surge in e-commerce has granted retailers unparalleled opportunities to analyze and cater to individual customer preferences. Thus, the personalized/contextual dynamic pricing has been adopted or considered in several key industries, including air travel, hotel reservations, insurance, and ride-sharing. The burgeoning practice within the industry has sparked a wealth of insightful research endeavors (see, e.g., \citet{Aydin2009, javanmard2016dynamic, ban2017personalized, lobel2018multidimensional, chen2018nonparametric,Fan2022Policy, Chen:22:privacy,chen2022statistical,Luo2023Distribution,Wang2023Online, chen2023privacy}).
	
	
	Our research distinguishes itself by integrating fairness constraints into the realm of personalized dynamic pricing, uncovering that the optimal regret takes the rate of $O(T^{2/3})$, instead of the standard rate $O(T^{1/2})$ \citep{rusmevichientong2010linearly,li2011pricing}. 
	Our lower bound theorem  validates the minimax regret, thereby underscoring the intrinsic challenges posed by the incorporation of fairness considerations into dynamic pricing models.

	\medskip
	\noindent \textbf{Fairness in pricing.} 
	Fairness, extensively studied in economics, has recently gained significant attention in operations management. 	Fairness has been incorporated into operations problems, such as  service priority, ads market resource allocation, online resource allocation, and dynamic rationing (see, e.g., \cite{chen2018why,Bateni:16,Chen:22:fairer,Balseiro:21:regularized,Manshadi:23:fair}). In pricing literature, fairness notion has been incorporated in to game-theoretical models for monopoly and duopoly markets (\cite{Li:16:behavior,Kallus:21,Cohen:22:price}).
	
Due to space constraints, an exhaustive review of fairness in operations management is beyond our scope. Instead, our discussion will center on a select few recent studies in fairness-aware dynamic pricing, noting that these studies have yet to integrate contextual or personalized information. The works by \citet{Cohen:22:price,Cohen:21:fairness} address group fairness by segmenting customers into distinct groups, ensuring that the price gap between groups \(i\) and \(j\) at any time \(t\) (represented as \(p_{i,t}\) and \(p_{j,t}\), respectively) does not exceed a predefined threshold \(\delta_{i,j}\) (i.e., \(|p_{i,t} - p_{j,t}| \leq \delta_{i,j}\)). In this work, we significantly expand on the existing works by modeling users' groups \emph{continuously}, with both user contexts and their utility values being continuous real numbers
and the pairwise group fairness notion being superseded with Lipschitz continuity constraints that apply to continuous values.
Moving beyond a static \(\delta_{i,j}\), \citet{Chen:2021:fair} introduced a relative-gap fairness constraint and delves into non-parametric demand functions. \citet{Xu:23:doubly} considered the dynamic pricing under two types of fairness constraint simultaneously, a procedural fairness and a substantive fairness , and established the $O(\sqrt{T})$ regret bound. Further, \citet{Chen:23:network} investigated dynamic network revenue management, employing the regularized fairness concept from \cite{Balseiro:21:regularized} to manage resource consumption during dynamic pricing. However, none of these studies on dynamic pricing with demand learning incorporate customer contextual information, which is the focal point of our paper.

	\medskip
	\noindent\textbf{Super-$\sqrt{T}$ Lower bound.} One of the important contributions made in this paper is to show that, when subject to utility fairness constraints,
	no learning-while-doing (bandit) algorithm could achieve an asymptotic regret significantly smaller than $O(T^{2/3})$, which contrasts most bandit and pricing settings
	in which $\tilde O(\sqrt{T})$ or better regret is expected. In the literature, there are several other bandit settings that have super-$\sqrt{T}$ lower bounds. Here we do a brief survey 
	of them and explain that why the reasons for such lower bounds in existing works are different from our settings. 
	
	When the underlying demand model or other mathematical objects needed to be learnt does \emph{not} have simple parametric forms, and the 
	overall problem lacks global convexity/concavity, an $\Omega(T^{3/5})$ or $\Omega(T^{3/4})$ regret lower bound might arise from the necessity of local smooth approximations.
	This is the case in the works of \cite{wang2021multimodal,chen2023optimal}, and more generally the continuum-armed (contextual) bandit literature (see for example \citep{bubeck2011x,gur2022smoothness,hu2022smooth}). Because the demand model studied in this paper is a generalized linear model with finite number of parameters
	and general regular conditions, the $\Omega(T^{2/3})$ regret bound established here is \emph{not} from the lack of parametric forms.
	
	When the application scenario limits the number of \emph{action changes}, or place significant penalties on certain action switches (e.g.~price protection, markdown pricing, or expensive capacity adjustments),
	a super-$\sqrt{T}$ regret lower bound might arise because of the bandit algorithm's inability to effectively explore certain arms.
	This is the case in the works of \citep{jia2021markdown,feng2023temporal,chen2023capacity,chen2020data,simchi2019phase},
	where the lower bounds range take the forms of $\Omega(T^{2/3})$ or a term polynomially higher than $\Omega(\sqrt{T})$, depending on the number of action changes/switches allowed.
	While similar at the surface, such lower bounds are fundamentally different from our results because in our problem there is no limitation or penalty at all on how frequently the bandit algorithm
	could change its prices or pricing policies over time.
	
	Indeed, an $\Omega(T^{2/3})$ regret lower bound arises in our setting because of the particular structure of the optimal fairness-constrained contextual pricing policy that is
	completely linear in the contexts, so that a bandit algorithm must make trade-offs between exploration of sub-optimal prices, or exploitation which inevitably introduces co-linearity among covariates.
	Such a phenomenon is quite rare in bandit related pricing literature and our lower bound result and proof techniques could potentially see applications to 
	other fairness-constrained bandit problems as well.

	

	
	\label{page:fixed}

	\section{Models and Assumptions}\label{sec:model}

	We adopt a generalized linear model to model the probabilistic realizations of demands given customer contexts $x\in\mX	\subseteq\mathbb R^d$ and offered price $p\in[\underline p,\overline p]\in[0,1]$.
	More specifically, the realized demand $y\in[0,1]$ is being modeled by
\begin{align}
\mathbb E[y|x,p] = f(x^\top\theta_0-\alpha_0 p),
\label{eq:glm-model}
\end{align}
where $f(\cdot)$ is a known ``link'' function. Examples of popular link functions include the identity function $f(u)=u$ corresponding to the linear demand model,
the Logistic function $f(u)=e^u/(1+e^u)$ corresponding to the Logistic demand model, the exponential function $f(u)=1-e^{-u}$ corresponding to the exponential demand model, etc.

	Let
	\begin{equation}
	r_u(p) = pf(u-\alpha_0 p)
\label{eq:defn-rup}
	\end{equation}
	be the expected revenue function for a customer with baseline utility value $u=x^\top\theta_0$.
We impose the following technical assumptions throughout this paper.
\begin{assumption}[Boundedness]
There exist constants $B\leq\tilde B<\infty$ such that for all $x,p\in\mX\times[\underline p,\overline p]$, $x^\top\theta_0\in[-B,B]$, $x^\top\theta_0-\alpha_0p\in [-\tilde B,\tilde B]$ and $f(x^\top\theta_0-\alpha_0 p)\in[0,1]$.
\label{asmp:boundedness}
\end{assumption}

\begin{assumption}[Smooth and strongly unimodal revenue]
\label{asmp:holder-sc}
There exist constants $L_f<\infty$ and $\sigma_r>0$ such that $|f(u)|,|f'(u)|,|f''(u)|\leq L_f$ for all $u\in[-\tilde B,\tilde B]$.
Furthermore, for every $u\in[-B,B]$, $r_u(\cdot)$ is uni-modal and there exists a unique strictly interior revenue-maximizing price $p^*(u)=\arg\min_{p\in[\underline p,\overline p]}r_u(p)\in(\underline p,\overline p)$
such that for every $p'\in[\underline p,\overline p]$, $r_u(p')\leq r_u(p^*(u)) - \frac{\sigma_r}{2}(p'-p^*(u))^2$.
\end{assumption}

\begin{assumption}[Stochastic contexts]
There exists an unknown underlying distribution $P_{\mX}$ supported on $\mX$, such that the context vector $x$ of a customer is distributed independently and identically from $P_{\mX}$,
and that the baseline utility values $x^\top\theta_0$ are supported on a compact interval contained in $[-B,B]$.
Furthermore, there exists a numerical constant $L_U<\infty$ such that for any measurable $U\subseteq[-B,B]$, $\Pr[x^\top\theta_0\in U]\leq L_U|U|/(2B)$, where $|U|$ is the Lesbesgue measure
of $U\subseteq[-B,B]$.
\label{asmp:stochasticity}
\end{assumption}

Assumption \ref{asmp:boundedness} assumes all model parameters are upper bounded by suitable constants, which is standard in the literature.
Note that some of the upper bound parameters here, such as $B$ or $\tilde B$, may potentially scale polynomially with dimension $d$. Because we treat dimension as a constant
in the low-dimensional setting studied in this paper, we shall not further elaborate on such potential dependency and instead deal with constants $B$ and $\tilde B$ throughout.

Assumption \ref{asmp:holder-sc} assumes that the link function $f$ is twice continuously differentiable with uniformly upper bounded derivatives.
Because the expected revenue function admits the form in Eq.~(\ref{eq:defn-rup}), this together with Assumption \ref{asmp:boundedness} implies that $r_u(\cdot)$ is twice continuously differentiable with bounded derivatives too.
Such smoothness assumption is standard in the literature.
Furthermore, Assumption \ref{asmp:holder-sc} assumes that, for every baseline utility value $u$, the expected revenue function $r_u(p)$ has a unique maximimzer $p^*(u)$ and decays quadratically as the price deviates from $p^*(u)$
on both directions. Such an assumption is implied by concavity in expected demand rates and monotonicity of link functions, which are satisfied by most popular demand functions (linear, Logisitc, exponential, etc.) \citep{li2011pricing}.
Notably, Assumption \ref{asmp:holder-sc} does \emph{not} imply that $r_u(\cdot)$ is concave in price, a condition that is usually violated for non-linear link functions.

Assumption \ref{asmp:stochasticity} assumes that users' context vectors $x\in\mX$ are independently and identically distributed with respect to an underlying distribution $P_{\mX}$ supported on $\mX$.
While in general contextual bandit the context vectors do not have to be stochastic and could even be adaptively adversarial, stochasticity of contexts is necessary here to define the optimal pricing policy subject to utility fairness constraints,
because there is no longer pointwise optimal prices when the entire pricing policy is subject to fairness constraints.
The second part of Assumption \ref{asmp:stochasticity} further assumes that the marginal distribution of $x^\top\theta_0$ admits a well-defined probability density function, which is a mild condition
and is satisfied provided that $P_{\mX}$ does not have mass concentrated on a measure-zero subset of $\mX$.

\subsection{Fair pricing policies}

In this section we give the formal mathematical definition of (utility) fairness for a certain pricing policy.
Let $\pi:\mX\to[\underline p,\overline p]$ be a pricing policy that maps any continuous context vector $x\in\mX$ to a price decision $p\in[\underline p,\overline p]$.
We shall restrict ourselves to deterministic pricing policies only in this paper, because such policies are optimal when the demand model in Eq.~(\ref{eq:glm-model}) is completely known and specified.
The following definition specifies the classes of fair contextual pricing policies:
\begin{definition}[utility fairness]
A contextual pricing policy $\pi:\mX\to[\underline p,\overline p]$ satisfies $\delta_0$-utility fairness ($\delta_0$-UF) for a small positive parameter $\delta_0>0$,
if for any $x,x'\in\mX$ it holds that $|\pi(x)-\pi(x')|\leq\delta_0|(x-x')^\top\theta_0|$.
\label{defn:fairness}
\end{definition}

Intuitively, Definition \ref{defn:fairness} requires the pricing policy $\pi$ to offer similar prices to customers whose baseline utility values $x^\top\theta_0$ (excluding pricing effects) are similar,
even though their context vectors or features $x,x'$ could be very different.
A pricing policy satisfying such fairness constraints will \emph{not} discriminate against specific context components or features of a customer, and furthermore explicitly set similar prices for customers with similar baseline utility values.


	\section{Optimal contextual pricing policies with fairness constraints}\label{sec:fullinfo}

	Given the definition of fairness constraints in the previous section, the optimal pricing policy that maximizes the expected revenue can be found by solving the following infinite-dimensional
	optimization problem:
	\begin{align}
	\pi^* &= \arg\max_{\pi:\mX\to[\underline p,\overline p]} \mathbb E_{x\sim P_{\mX}}[r_{x^\top\theta_0}(\pi(x))] 
	\label{eq:inf-opt-ver1}\\
	s.t.&\;\;\big|\pi(x)-\pi(x')\big|\leq \delta_0\big|(x-x')^\top\theta_0\big|.\;\;\forall x,x'\in \mX,\nonumber
	\end{align}
	Noting that both the objective and the right-hand side of the fairness constraints only involve the baseline utility value $u=x^\top\theta_0$, the optimal policy $\pi^*$ can be re-formulated
	into $\pi^*(x)=\varpi^*(x^\top\theta_0)$ such that
	\begin{align}
	\varpi^* &= \arg\max_{\varpi:[-B,B]\to[\underline p,\overline p]} \int_{-B}^B r_u(\varpi(u))\ud P_{\mU}(u)	\label{eq:inf-opt-ver2}\\
	s.t.&\;\; \big|\varpi(u)-\varpi(u')\big|\leq\delta_0\big|u-u'\big|, \;\;\forall u,u'\in [-B,B],\nonumber
	\end{align}
	where $P_{\mU}$ is the marginal distribution of the baseline utility $u=x^\top\theta_0$, whose support is contained in $[-B,B]$ thanks to Assumption \ref{asmp:boundedness}.
	Note that Eq.~(\ref{eq:inf-opt-ver2}) remains an infinite-dimensional optimization problem in search of an optimal function,
	but the fact that $\varpi^*$ is a univariate function enables us to find approximate computational procedures and also to gain deeper insights into the structures of the optimal policy.

	\subsection{Discretization and an approximation algorithm for finding the optimal policy}\label{sec:fptas}
	
	We first present a general-purpose approximation algorithm based on discretization that approximately solves the infinite-dimensional problem in Eq.~(\ref{eq:inf-opt-ver2}).
	Let $\epsilon>0$ be a small positive parameter for approximation errors and $\{u_j\}_{j=1}^{M_u}$, $\{p_j\}_{j=1}^{M_p}$ be discretized base utility values and prices,
	where $M_u=\lceil 2B/\epsilon\rceil$, $M_p=\lceil (\overline p-\underline p)/(\delta_0\epsilon)\rceil$, and $u_{j+1}=u_j+\epsilon$, $p_{j+1}=p_j+\delta_0\epsilon$ for all $j$, such that
	$[-B,B]=\bigcup_{j=1}^{M_u}[u_j-\epsilon/2,u_j+\epsilon/2]$ and $[\underline p,\overline p]=\bigcup_{j=1}^{M_p}[p_j-\delta_0\epsilon/2,p_j+\delta_0\epsilon/2]$.
	Note that the grids $\{u_j\},\{p_j\}$ are both depending on the discretization parameter $\epsilon$: we are omitting $\epsilon$ subscripts to simplify notations but the fact that they are $\epsilon$-dependent
	must be kept in mind.
	Consider the following finite-dimensional, discrete optimization problem:
	\begin{align}
	j_1^*,\cdots, j_{M_u}^* &= \arg\max_{j_1,\cdots,j_{M_u}\in[M_p]} \sum_{k=1}^{M_u} \gamma_k r_{u_k}(p_{j_k}) \label{eq:inf-opt-ver3}\\
	s.t.&\;\; |j_{k+1}-j_k|\leq 1, \;\;\;\; k=1,2,\cdots,M_u-1,\nonumber
	\end{align}
	where $\gamma_k := \Pr[x^\top\theta_0 \in [u_k-\epsilon/2,u_k+\epsilon/2]] = \int_{u_k-\epsilon/2}^{u_k+\epsilon/2}\ud P_\mU(u)$ is the probability of the baseline utility falling into $[u_k-\epsilon/2,u_k+\epsilon/2]$. With solutions $j_1^*,\cdots,j_{M_u}^*\in[M_p]$, an approximate contextual pricing policy could be constructed as $\pi_\epsilon^*(x)=\varpi_\epsilon^*(x^\top\theta_0)$ with
	\begin{align}
	\varpi_\epsilon^*(u) = p_{j_{k(u)}^*} + \delta_0\iota_{k(u)}(u-u_{k(x)};\{j_k^*\}) \;\;\;\;\;\;\text{where}\;\; k(u) = \arg\min_{k\in[M_u]} \big|u_k-u\big|,
	\label{eq:pistar-epsilon}
	\end{align}
	where ties are broken arbitrarily in the argmax operator above, and the $\iota_{k}(\cdot)\in\{\pm 1\}$ functions are defined in the following ways:
\begin{align}
\iota_{k}(w;\{j_k\}) := \left\{\begin{array}{lr}+w,& \text{if $w>0$ and $j_{k+1}=j_k+1$, or $w<0$ and $j_{k-1}=j_k-1$;}\\
-w,& \text{if $w>0$ and $j_{k+1}=j_k-1$, or $w<0$ and $j_{k-1}=j_k+1$};\\ 0,& \text{otherwise.}\end{array}\right.\label{eq:defn-iota}
\end{align}
Intuitively, the constructed $\varpi_\epsilon^*(\cdot)$ is a piecewise linear interpolation of $\{u_k,p_{j_k^*}\}_k$, which must be $\delta_0$-Lipschitz continuous
thanks to the feasibility of the discrete price solutions.
Figure \ref{fig:varpi_eps_star} gives a graphical illustration of the constructed $\varpi_\epsilon^*$ and $\pi_\epsilon^*$ policy.

\begin{figure}[t]
\centering
\includegraphics[width=0.7\textwidth]{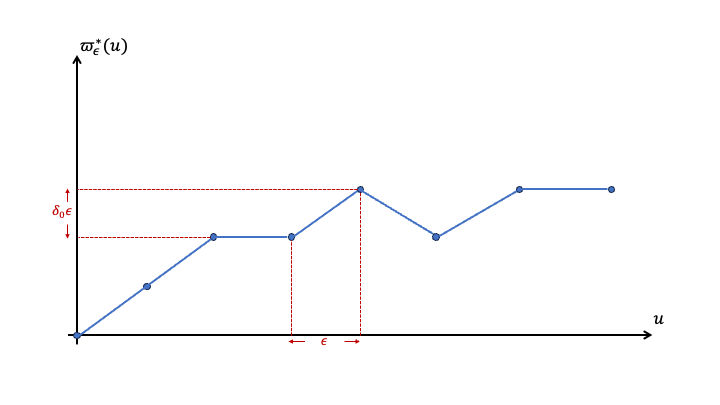}
\caption{Illustration of the construction of $\varpi_\epsilon^*(\cdot)$ from $\{j_k^*\}$.}
\label{fig:varpi_eps_star}
{\small Note: constructed $\varpi_\epsilon^*(\cdot)$ corresponding to the sequence $j_1^*=1$, $j_2^*=2$, $j_3^*=3$, $j_4^*=3$, $j_5^*4$, $j_6^*=3$, $j_7^*=4$, $j_8^*=4$.
Gaps between two consecutive points are $\epsilon$ on the X-axis and $\delta_0\epsilon$ on the Y-axis.
$\varpi_\epsilon^*(\cdot)$ is then a piecewise linear interpolation of all discrete utility-price points, which is $\delta_0$-Lipschitz continuous thanks to the constraints on $\{j_k^*\}$.
Finally, the implied optimal pricing policy is $\pi_\epsilon^*(x)=\varpi_\epsilon^*(x^\top\theta_0)$.
Plot for illustration purposes only and the plotted $\{j_k^*\}$ or $\varpi_\epsilon^*(\cdot)$ may not correspond to the optimal solution of an actual problem instance.}
\end{figure}

	The following proposition shows that $\pi_\epsilon^*$ is fair and near optimal.
	\begin{lemma}
	The policy $\pi_\epsilon^*$ satisfies $\delta_0$-utility fairness constraint. Furthermore, it holds that
	$$
	\mathbb E_{x\sim P_{\mX}}[r_{x^\top\theta_0}(\pi_\epsilon^*(x))] \geq \mathbb E_{x\sim P_{\mX}}[r_{x^\top\theta_0}(\pi^*(x))] - 4L_f\delta_0\epsilon.
	$$
	\label{prop:approx}
	\end{lemma}
	Lemma \ref{prop:approx} shows that, if we can efficiently compute the optimal solution to Eq.~(\ref{eq:pistar-epsilon}) for \emph{every} fixed $\epsilon>0$, and that the computational complexity is polyonomial on both the problem
sizer and $1/\epsilon$,
then by taking $\epsilon\to 0^+$ we obtain a (fully) polynomial-time approximation scheme (FPTAS) of the infinite-dimensional optimization problem in Eq.~(\ref{eq:inf-opt-ver2}).
Below we describe how such an efficient computation algorithm could be designed, by using a dynamic programming approach.

	In particular, for $k\in[M_u]$ and $j\in[M_p]$, let $V(k,j)$ be the maximized expected revenue up to $u_k$ such that the last price is $j_k=j$.
	The Bellman's equation can then be written as
	\begin{equation}
	V(k+1,j) = \gamma_{k+1} r_{u_k}(p_j) + \max_{|j'-j|\leq 1} V(k,j'),
	\label{eq:dp-fptas}
	\end{equation}
	with the initial condition that $V(1,j)=\gamma_1r_{u_1}(p_j)$ for every $j$.
	Solving the dynamic program requires $O(\delta_0^{-1}\epsilon^{-2})$ space and time complexity, which achieves an $O(\epsilon)$-approximation
	thanks to Lemma \ref{prop:approx}.
	
	To conclude this section, we present an additional lemma that complements Lemma \ref{prop:approx} by showing that, if the optimal $\delta_0$-UF policy $\pi^*$ is unique,
	then any $\delta_0$-UF policy sequence that converges to $\pi^*$ in expected revenue (which applies to $\{\pi_\epsilon^*\}$ constructed by discretization)
	also converges to $\pi^*$ uniformly at each utility value. While not immediately useful from a computational perspective,
	this additional result gives us more insights into the discretization approximation procedure, and is important in understanding the structures of optimal policy in the next section.
	
	To describe the lemma we first define some notations. Let 
	$$
	\mF := \big\{ g:[-B,B]\to[\underline p,\overline p]\big| |g(u)-g(u')|\leq \delta_0|u-u'|,\forall u,u'\in[-B,B]\big\}
	$$
	be the class of all functions mapping from $[-B,B]$ to $[\underline p,\overline p]$ that are $\delta_0$-Lipschitz continuous.
	It is clear that for any utility-based contextual pricing policy $\pi(x)=\varpi(x^\top\theta_0)$, $\pi$ satisfies $\delta_0$-UF if and only if $\varpi\in\mF$.
	For any $g,h\in\mF$, define 
	$$
	\|g-h\|_{\infty} := \sup_{u\in[-B,B]}\big|g(u)-h(u)\big|.
	$$
	We then have the following lemma:
	\begin{lemma}
	Suppose the optimal $\delta_0$-UF contextual pricing policy $\pi^*(x)=\varpi^*(x^\top\theta_0)$ is unique.
	Let $\{\varpi_n\}_{n\in\mathbb N}\subseteq\mF$ be a sequence of functions in $\mF$ such that 
	$$
	\lim_{n\to\infty} \mathbb E_{P_{\mU}}[r_u(\varpi_n(u))] = \mathbb E_{P_{\mU}}[r_u(\varpi^*(u))].
	$$
	Then it holds that
	$$
	\lim_{n\to\infty} \|\varpi_n-\varpi^*\|_{\infty} = 0.
	$$
	\label{lem:approx-uniform}
	\end{lemma}
	While seemingly intuitive, the proof of Lemma \ref{lem:approx-uniform} is non-trivial and given in the supplementary material.
{The main idea behind the proof of Lemma \ref{lem:approx-uniform} is the following: the set of all $\delta_0$-UF contextual pricing policies can be completely characterized by the set of all $\delta_0$-Lipschitz continuous
functions on $[0,1]$. By exclusing a small neighborhood around $\varpi^*$ and using the Arzela-Ascoli theorem, the remaining function set is compact with respect to the point-wise convergence norm $\|\cdot\|_{\infty}$.
On the other hand, the expected revenue of any function is continuous in $\|\cdot\|_{\infty}$. This means that the supreme of the maximal expected revenue can be attained on a compact space which is different from the objective
of $\varpi^*$ thanks to the uniqueness of $\varpi^*$.}
	
	Finally, our next lemma shows that the expected revenue of the optimal policy is a left-continuous function of the fairness parameter $\delta_0$.
	\begin{lemma}
	For $\delta_0>0$, let $R(\delta_0)$ be the expected revenue of optimal $\delta_0$-UF policy. Then for every $\delta_0>0$, $\lim_{\delta\to\delta_0^-}R(\delta)=R(\delta_0)$.
	\label{lem:optimal-obj-continuity}
	\end{lemma}
Due to space constraints, the complete proof of Lemma \ref{lem:optimal-obj-continuity} is again deferred to the supplementary material.
The proof is at a higher level similar to the proof of Lemma \ref{lem:approx-uniform}, utilizing pointwise convergence of a sequence of pricing policies and the compactness of $\delta_0$-UF policies.
		
\subsection{Structures of the optimal fair contextual pricing policy}
\label{sec:structure}

In addition to the discretization and approximation algorithm proposed in the previous section that is general-purpose and leads to an algorithm that is practical in most scenarios,
in this section we present additional structures in the optimal fair contextual pricing policy that offer more insights into the optimal policy,
while laying the foundation of follow-up analysis of the cost of fairness as well as a pricing policy with \emph{incomplete information} that is more useful in practice when
the demand curve is unknown in advance.

To derive the structures of the optimal fair pricing policy, we impose some additional assumptions on the link function $f$. In particular, we assume the existence of a parameter 
\begin{equation}
\sigma_u := \inf_{u\in[-\tilde B,\tilde B]}\left\{f'(u) - \frac{\alpha_0\overline p}{2}\big|f''(u)\big|\right\} > 0
\label{eq:defn-sigma-u}
\end{equation}
that is strictly positive.
To understand the high-level intuition behind Eq.~(\ref{eq:defn-sigma-u}) (or more specifically the assumption that $\sigma_u>0$, since a $\sigma_u$ such defined as the infimum always exists),
observe that an economically reasonable link function must satisfy $f'(u)>0$ for all $u$ because an increased utility means higher purchase probability, and therefore higher expected demand.
Therefore, Eq.~(\ref{eq:defn-sigma-u}) is essentially assuming that either the price elasticity $\alpha_0$ or the non-linearity of the demand curve $f''$ is small. Below, we use several examples
of popular demand functions to further illustrate the $\sigma_u$ parameter.
\begin{example}[Linear demand]
When the demand function is linear we have the identity link function $f(u)=u$. Because the demand curve is perfectly linear and $f''\equiv 0$ in this case, 
$\sigma_u=1$ regardless of the value of price elasticity $\alpha_0$.
\end{example}
\begin{example}[Logistic demand]
When the demand function is Logistic we have $f(u)=e^u/(1+e^u)$. Eq.~(\ref{eq:defn-sigma-u}) holds for every $f^{-1}(\max\{0,\frac{1}{2}-\frac{1}{\alpha_0\bar p}\})<u<f^{-1}(\min\{1,\frac{1}{2}+\frac{1}{\alpha_0\bar p}\})$,
where $f^{-1}$ is the inverse function of $f$. Furthermore, if $\alpha_0\leq 2/\bar p$ then Eq.~(\ref{eq:defn-sigma-u}) holds for every $u\in\mathbb R$.
\end{example}
\begin{example}[Exponential demand]
When the demand function is exponential we have $f(u)=1-e^{-u}$ for $u\geq 0$. Eq.~(\ref{eq:defn-sigma-u}) holds for every $u\geq 0$ provided that $\alpha_0<2/\bar p$.
\end{example}

Another important motivation for considering Eq.~(\ref{eq:defn-sigma-u}) lies in the observation that, if $\sigma_u>0$, the optimal price (without subjecting to any fairness constraints)
is a \emph{monotonically increasing} function of the baseline utility $u=x^\top\theta_0$, which is economically intuitive and practically reasonable because we are expected to offer (even only slightly) higher prices
to individuals who value the offered products or service more. This intuition is rigorously summarized in the following lemma:
\begin{lemma}
Let $\sigma_u>0$ be defined in Eq.~(\ref{eq:defn-sigma-u}) and $M_r:=\sup_u |r''(u)|<\infty$, whose existence is implied by Assumption \ref{asmp:boundedness}.
For every $u\in[-B,B]$, let $p^*(u)=\arg\max_{p}r_u(p)$ which belongs to $(-\underline p,\overline p)$ thanks to Assumption \ref{asmp:holder-sc}. Then for every $u\leq u'$, it holds that
$$
p^*(u) + \frac{\sigma_u}{M_r} (u'-u)\leq p^*(u') \leq p^*(u) + \sqrt{\frac{4L_f}{\sigma_r}(u'-u)}.
$$
\label{lem:pstar-mono}
\end{lemma}
Lemma \ref{lem:pstar-mono} can be proved by taking partial derivatives of $r$ with respect to both $u$ and $p$, and then applying the mean-value theorem.
Due to space constraints the complete proof of Lemma \ref{lem:pstar-mono} is given in the supplementary material.

With $\sigma_u>0$ defined in Eq.~(\ref{eq:defn-sigma-u}), $M_r=\sup_u |r''(u)|<\infty$ as well as all assumptions listed in Sec.~\ref{sec:model},
our next main result establishes the structure of the optimal contextual pricing policy satisfying $\delta_0$-utility fairness.

\begin{theorem}
Suppose that $\sigma_u>0$ and the optimal contextual pricing policy $\pi^*(x)=\varpi^*(x^\top\theta_0)$ satisfying $\delta_0$-UF is unique. Then $\varpi^*$ satisfies the following properties:
\begin{enumerate}
\item $\varpi^*$ is monotonically non-decreasing; that is, for every $u\leq u'$ it holds that $\varpi^*(u)\leq \varpi^*(u')$;
\item $\varpi^*$ is differentiable almost everywhere, and $\varpi^*(u)=p^*(u)$ if $\varpi^*$ is differentiable at $u\in(-B,B)$ and $\partial_u\varpi^*(u)<\delta_0$.
\end{enumerate}
Furthermore, if $\delta_0>0$ satisfies $\delta_0\leq\sigma_u/M_r$ then $\varpi^*$ takes the form of
\begin{equation}
\varpi^*(u) = \pi_0 + \delta_0 u,
\label{eq:optimal-linear-form}
\end{equation}
where $\pi_0=\varpi^*(0)\in[\underline p,\overline p]$ is the price set for the case of zero baseline utility value.
\label{thm:optimal-pi}
\end{theorem}

Before discussing the consequences and proof of Theorem \ref{thm:optimal-pi}, we make a remark when the demand model $f(u)=u$ is the purely linear model.
In particulr, the following remark shows that for purely linear models, Theorem \ref{thm:optimal-pi} compeletely characterizes the behavior/structure of the optimal $\delta_0$-UF contextual pricing policy
for \emph{every} value of $\delta_0>0$.
\begin{remark}
When the demand curve is linear, Theorem \ref{thm:optimal-pi} completely characterizes the structure of the optimal $\delta_0$-UF contextual pricing policy for \emph{all} values of $\delta_0$.
To see this, note that for a demand model of $f(u-\alpha_0 p)=u-\alpha_0p$, the optimal price (without fairness constraints) for a baseline utility value $u$ is $p^*(u)=u/\alpha_0$,
which automatically satisfies $\delta_0$-UF for if $\delta_0\geq 1/\alpha_0$. On the other hand, note that in this case $\sigma_r=1$, $M_r=\alpha_0$ and Theorem \ref{thm:optimal-pi} specifies the structure of the optimal
$\delta_0$-UF contextual pricing policy for all $\delta_0< 1/\alpha_0$. Combining both cases, we obtain
$$
\varpi^*(u)=\left\{\begin{array}{ll}\varpi^*(0)+\delta_0 u,& \text{if }\delta_0<1/\alpha_0;\\ p^*(u),& \text{if }\delta_0\geq 1/\alpha_0.\end{array}\right.
$$
$\blacksquare$
\end{remark}

We next explain the conclusions and consequence of Theorem \ref{thm:optimal-pi}.
The first two properties establish two important structures of $\varpi^*$ and they apply to all values of $\delta_0>0$.
Intuitively, $\varpi^*$ always sets the offered price as a monotonically increasing function
of the baseline utility value, much in the same way of the fairness unconstrained policy in this case since $p^*(u)$ is a monotonically increasing function of $u$,
thanks to Lemma \ref{lem:pstar-mono}.
Additionally, if the slope of $\varpi^*$ is strictly less than $\delta_0$ at a certain point,\footnote{The slope of $\varpi^*$ can never exceeds $\delta_0$ because of the definition of $\delta_0$-utility fairness.}
then the offered price under $\varpi^*$ must coincide with $p^*(u)$, the optimal price without fairness considerations.
This shows that, roughly speaking, $\varpi^*$ either increases at a rate/slope of $\delta_0$, the maximum allowed under $\delta_0$-utility fairness,
or $\varpi^*$ must take the form of \emph{unconstrained} contextual pricing when the utility fairness constraint is ``not binding'' at a certain point.

Theorem \ref{thm:optimal-pi} also shows that, when the utility fairness parameter $\delta_0$ is not too large, the optimal contextual pricing policy $\pi^*(x)=\varpi^*(x^\top\theta_0)$
takes a simple form of offering prices linearly increasing with the baseline utility value, with the slope being set exactly at $\delta_0$.

Next, we give a high-level sketch of the proof of Theorem \ref{thm:optimal-pi}.
The complete proof is very technical and therefore deferred to the supplementary material.
The idea is to discretize the infinite-dimensional optimization problem of $\varpi^*$, so that there are a finite number of (discretized) utility values.
The prices for each utility value, different from the formulation of Eq.~(\ref{eq:inf-opt-ver2}), can free to vary continuously in $[\underline p,\overline p]$.
This leads to a finite-dimensional, linear constrained optimization problem, in which strong duality holds thanks to Slater's condition.
We then use complementary slackness and induction to analyze the signs of Lagrangian multipliers of utility fairness constraints at optimality, which eventually leads to the conclusions given in Theorem \ref{thm:optimal-pi}.

\subsection{Cost of utility fairness}

In this section we study the ``cost'' of utility fairness constraints on the expected revenue when an optimal contextual pricing policy (subject to corresponding fairness constraints) is implemented.
More specifically, for every $\delta_0>0$, recall the definition that 
$$
R(\delta_0) := \max_{\varpi:[-B,B]\to[\underline p,\overline p]}\mathbb E_{P_{\mU}}[r_u(\varpi(u))] \;\;\;\;s.t.\;\; \big|\varpi(u)-\varpi(u')\big|\leq\delta_0\big|u-u'\big|,\;\;\forall u,u'\in[-B,B]
$$
which is the expected revenue of the optimal contextual pricing policy that satisfies $\delta_0$-UF.
Clearly, $R(\delta_0)$ is a monotonically non-decreasing function of $\delta_0$ because the feasible region gets larger as $\delta_0$ increases,
and $R(+\infty)$ is the expected revenue of the optimal contextual pricing policy \emph{without} any fairness constraints. We then define
$$
\rho(\delta_0) := \frac{R(\delta_0)}{R(+\infty)} \in [0,1]
$$
as the ratio between the optimal revenue at UF level $\delta_0$ versus the optimal revenue without fairness constraints.
This quantity characterizes the tradeoff between revenue and fairness constraints and also sheds light on the cost of utility fairness: a $\rho(\delta_0)$ value close to one means that
the $\delta_0$-UF constraint is not as penalizing compared with the revenue ideally obtainable without any privacy constraints, while on the other hand a $\rho(\delta_0)$ value close to zero
indicates that such UF requirement at the $\delta_0$ level is probably too restrictive, because even the optimal contextual pricing policy subject to such fairness constraints
suffers from a big gap from the ideal revenue without fairness constraints.

Our next theorem gives closed-form expression of $\rho(\delta_0)$ when the demand model is purely linear.
\begin{theorem}
Let $f(x)=x$ be the identity link function, corresponding to a purely linear demand model, and $\alpha_0>0$ being the price elasticity parameter. Then
$\rho(\delta_0)=1$ for all $\delta_0\geq1/\alpha_0$, and
$$
\rho(\delta_0) = \alpha_0\delta_0(2-\alpha_0\delta_0) + (1-\alpha_0\delta_0)^2\frac{\mu_u^2}{\nu_u^2},\;\;\;\;\;\;\forall\delta_0<\frac{1}{\alpha_0},
$$
\label{thm:cost-uf}
where $\mu_u=\mathbb E_{P_{\mU}}[u]$ and $\nu_u^2=\mathbb E_{P_{\mU}}[u^2]$.
\end{theorem}

\begin{remark}
Let $\kappa_u := \mu_u^2/\nu_u^2$. Because $\nu_u^2-\mu_u^2$ is the variance of $u$ under $P_{\mU}$ which is non-negative, we have that $\kappa_u\in[0,1]$.
Using elementary algebra calculations we have that $\rho'(\delta_0)=2\alpha_0(1-\kappa_u)(1-\alpha_0\delta_0)$, which is positive for all $\delta_0\in(0,1/\alpha_0)$,
compatible with the intuition and property that $\rho$ is a monotonically increasing function of $u$.
Furthermore, because $\rho''(\delta_0)=-2\alpha_0^2(1-\kappa_u)<0$ for all $\delta_0<1/\alpha_0$, the function $\rho$ is \emph{concave} in $\delta_0$ on $[0,1/\alpha_0]$.
\label{rem:rho-linear}
\end{remark}

In general, Theorem \ref{thm:cost-uf} and Remark \ref{rem:rho-linear} together paint the complete picture of the impact of $\delta_0$-utility fairness constraints the expected revenue of the (optimal) pricing policy.
As $\delta_0$ increases, the ``relative efficiency'' $\rho(\delta_0)$ gradually increases to one in a quadratic, concave way until $\delta_0$ hits the threshold of $1/\alpha_0$ at which $\rho(\delta_0)$ saturates at $\rho(\delta_0)=1$
for ever larger $\delta_0$ values (i.e.~weaker fairness constraints). Graphical illustrations of Theorem \ref{thm:cost-uf} are given in the next paragraph.

Theorem \ref{thm:cost-uf} only applies to the purely linear model $f(u)=u$. For more general non-linear demand models, the complexity of the expected revenue and optimal pricing strategies
makes analytical results such as Theorem \ref{thm:cost-uf} difficult to obtain. Nevertheless, we could still use numerical and computational studies to get some insights in this general case,
which is covered in the next paragraph.

\subsection{Computational results}

\begin{figure}[t]
\centering
\includegraphics[width=0.45\textwidth]{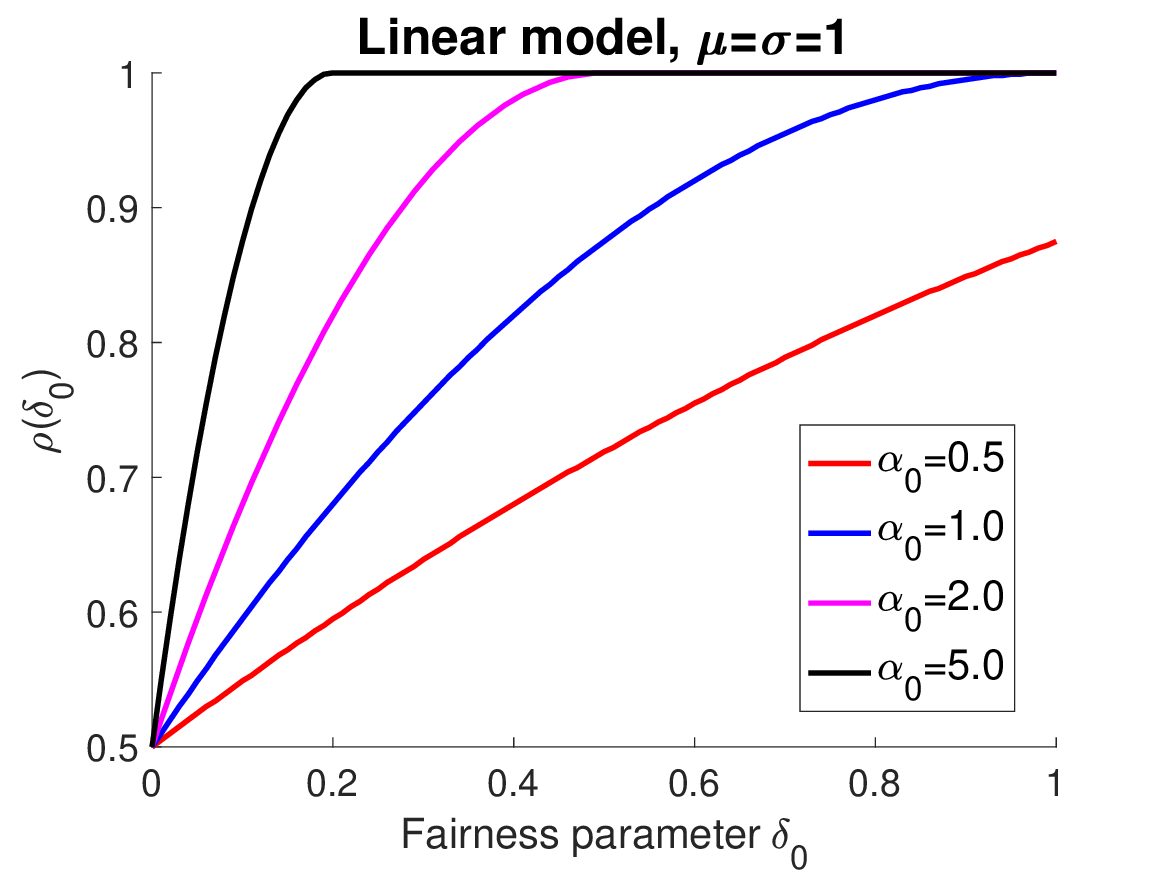}
\includegraphics[width=0.45\textwidth]{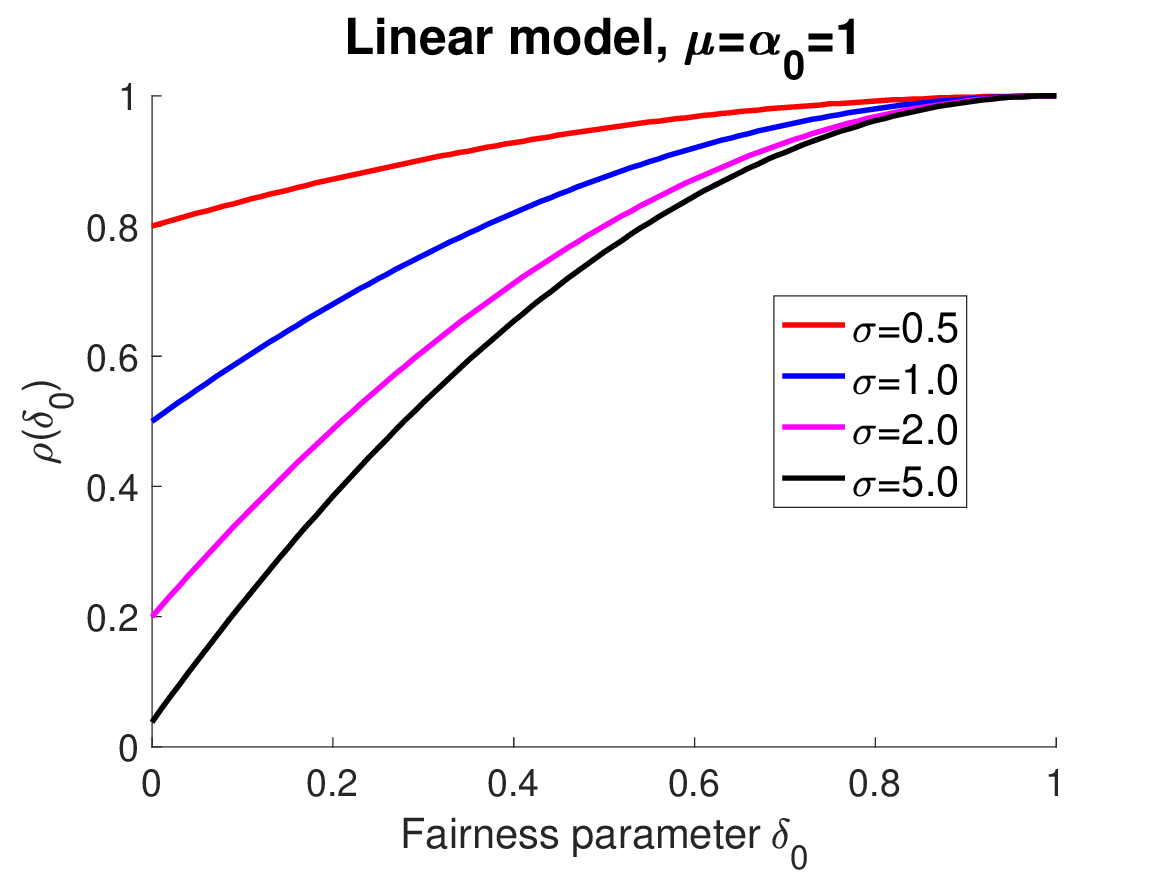}
\caption{Values of $\rho(\delta_0)$ for purely linear models $f(u)=u$.}
\label{fig:rho-linear}
{\small Note: in both figures, $\mu=\mathbb E_{P_\mX}[u]$ and $\sigma = \sqrt{\mathbb E_{P_{\mX}}[(u-\mu)^2]}$. By Theorem \ref{thm:cost-uf},
for purely linear models only the mean and standard deviation of the baseline utility values are important. The left figure fixes $\mu=\sigma=1$
and plots $\rho(\delta_0)$ as a function of $\delta_0$ for different values of $\alpha_0$, the price elasticity parameter.
The right figure fixes $\mu=\alpha_0=1$ and plots the $\rho(\cdot)$ curve for different values of the standard deviation (of the baseline utility), $\sigma$.}
\end{figure}

\begin{figure}[t]
\centering
\includegraphics[width=0.32\textwidth]{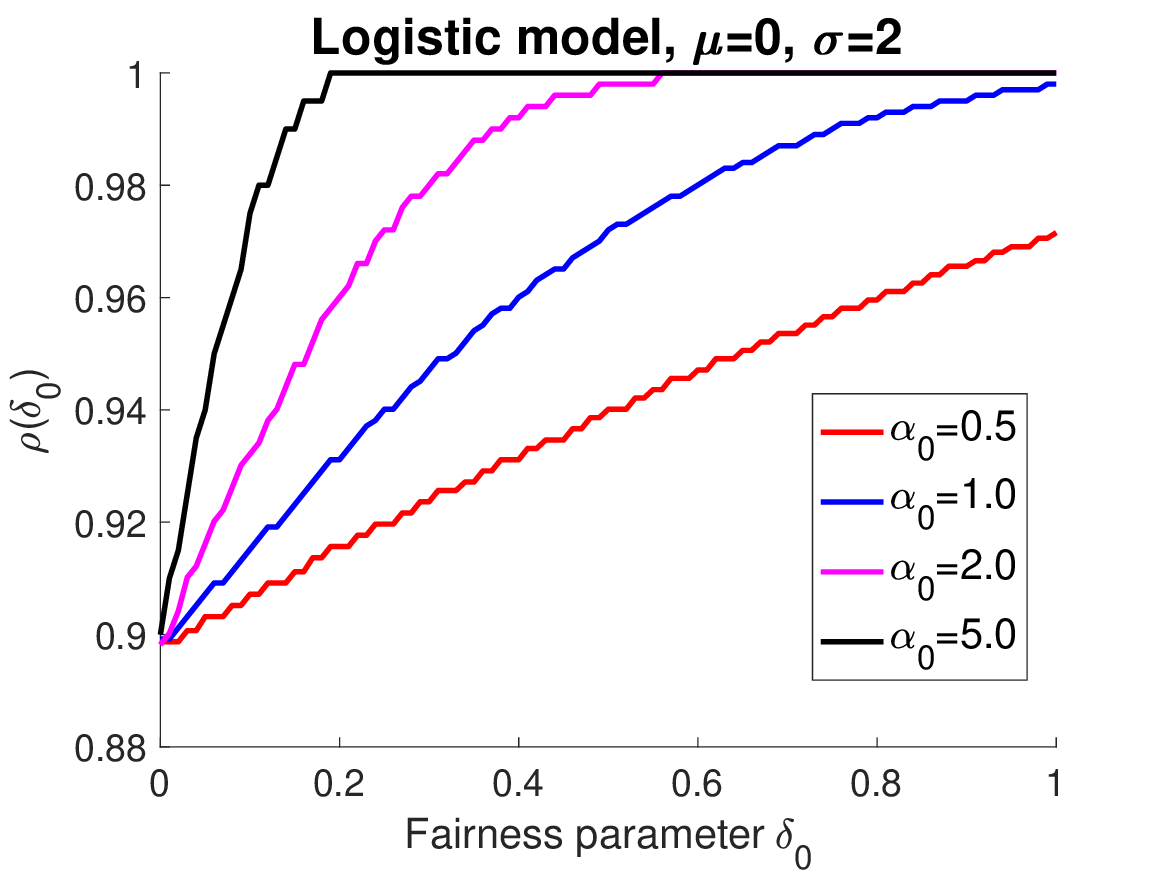}
\includegraphics[width=0.32\textwidth]{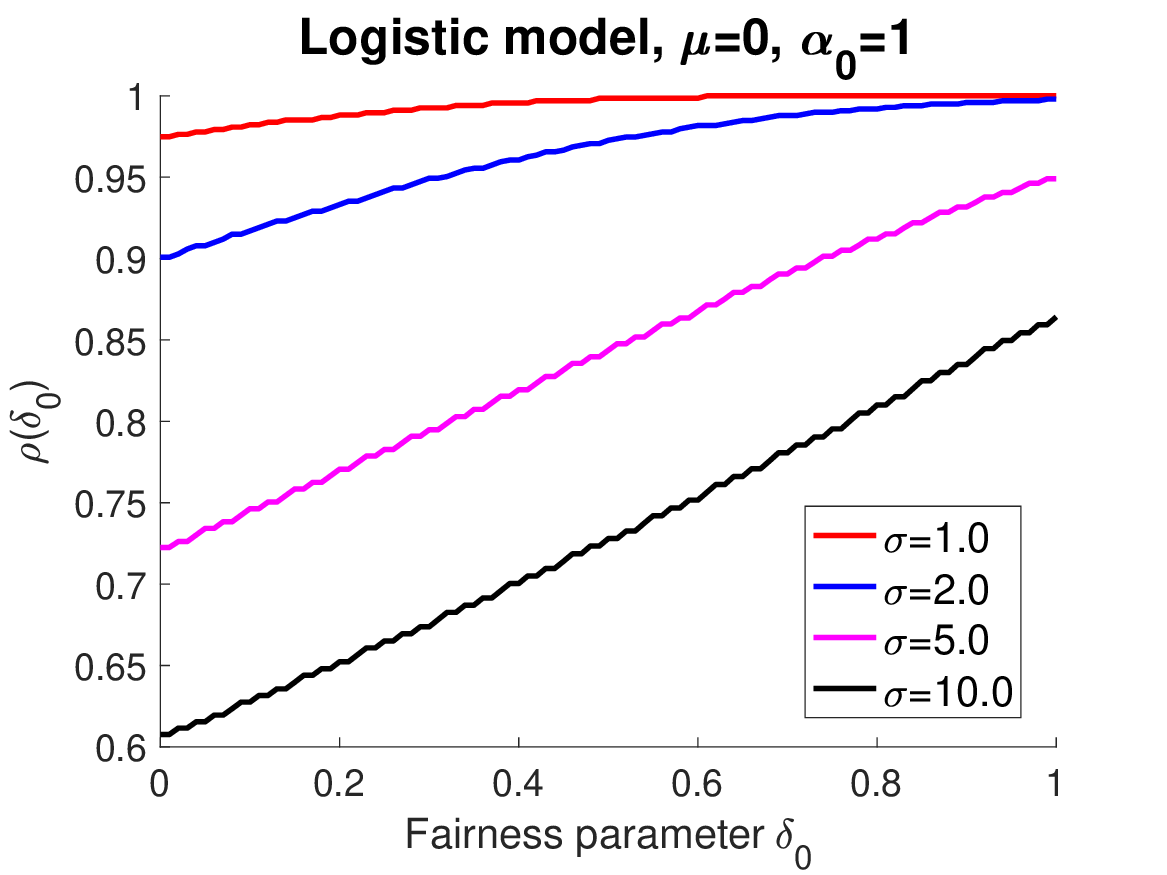}
\includegraphics[width=0.32\textwidth]{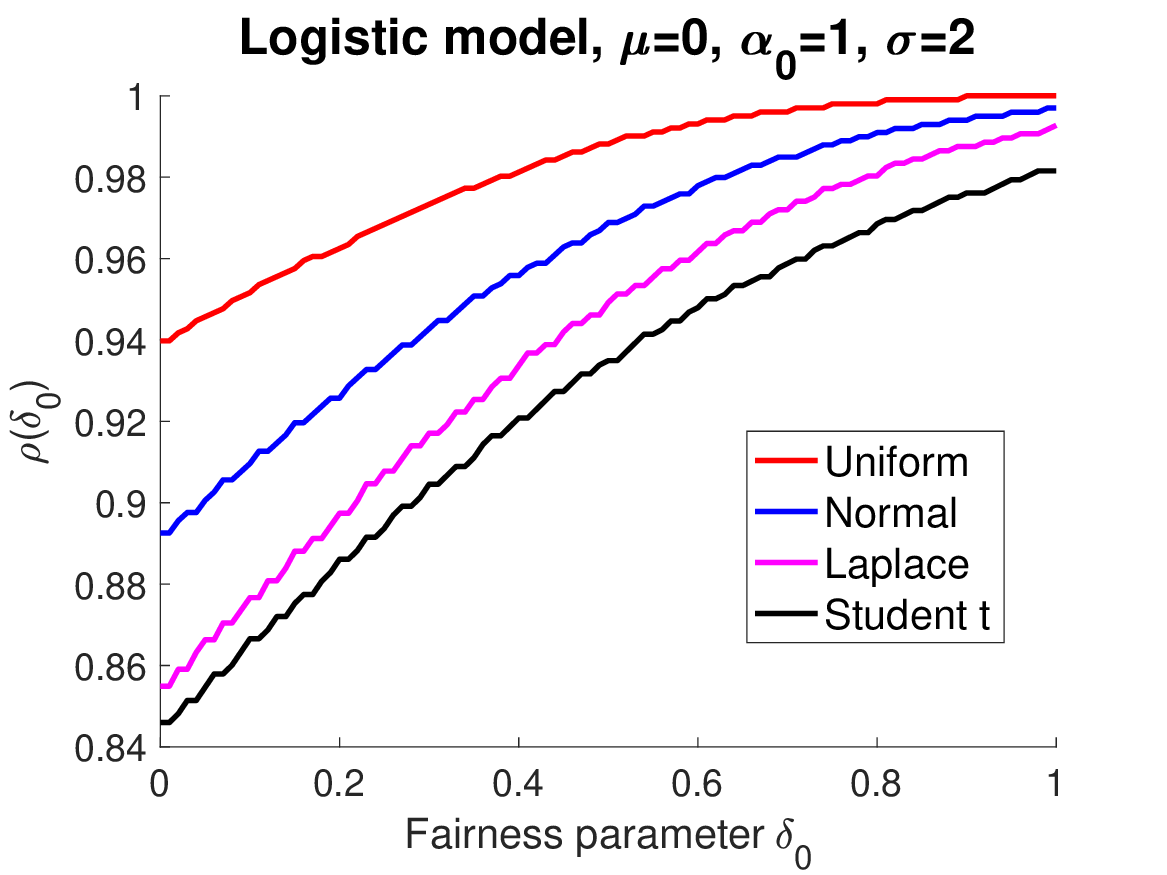}
\caption{Values of $\rho(\delta_0)$ for Logistic demand models $f(u)=e^u/(1+e^u)$.}
{\small Note: in both figures, $\mu=\mathbb E_{P_\mX}[u]$ and $\sigma = \sqrt{\mathbb E_{P_{\mX}}[(u-\mu)^2]}$.
Curves calculated using the FPTAS developed in Sec.~\ref{sec:fptas}, with further details in the main text.
Because approximate calculations are used, the curves have small perturbations in all plots.
The left panel plots $R(\delta_0)$ for different price elasticity parameters $\alpha_0$; the middle panel plots $R(\delta_0)$ for different standard deviations
of the underlying distribution. In both these plots, underlying distributions of baseline utility values are centered Normal.
The right panel plots $R(\delta_0)$ for fixed price elasticity, utility mean and utility standard deviation, but different types of distributions (all centered): 
the uniform distribution, the Normal distribution, the Laplace distribution, and the Student's t-distribution with degree-of-freedom being 3.
}
\label{fig:rho-logistic}
\end{figure}

We complement the results in the previous section with some computational results and graphical illustrations of the cost (in terms of reduced expected revenue) 
of utility fairness constraints.

In Figure \ref{fig:rho-linear}, two plots of the relative revenue performance (of the best pricing policy subject to $\delta_0$-UF) is plotted as a function of the fairness parameter $\delta_0$,
for the purely linear model $f(u)=u$.
These plots are produced by instantiating the formula for $\rho(\cdot)$ derived in Theorem \ref{thm:cost-uf}.
We observe that under various settings and scenarios, $\rho(\cdot)$ is a monotonically increasing, concave curve.
Furthermore, when the underlying distribution of context $P_{\mX}$ is fixed, the left panel of Figure \ref{fig:rho-linear} shows that larger price elasticity $\alpha_0$ leads to 
smaller ``transition points'' of the fairness parameter $\delta_0$ such that the revenue of a $\delta_0$-UF policy converges to that of a non-fair one at smaller levels of $\delta_0$,
 but steeper decays of the revenue when stricter fairness constraints are imposed ($\delta_0$ below the transition points).
The right panel of Figure \ref{fig:rho-linear} demonstrates the effects of different standard deviations of the underlying utility distribution, when fixing the mean of the utility distribution
and the price elasticity parameter $\alpha_0$ unchanged. In this case, when $\delta_0$ increases $R(\delta_0)$ converges to the same point for all scenarios,
but the decay of expected revenue is faster for distributions with larger standard deviation.

In Figure \ref{fig:rho-logistic}, three plots of the relative revenue performance (of the best pricing policy subject to $\delta_0$-UF) is plotted as a function of the fairness parameter $\delta_0$,
for the Logistic demand model $f(u)=e^u/(1+e^u)$.
These plots are produced by running the FPTAS dynamic programming developed in Sec.~\ref{sec:fptas}, with one hundred discretization points for utility values and ten thousand 
discretization points for prices. Note that, because of such discretization and approximation, calculations of $\rho(\delta_0)$ are approximate in Figure \ref{fig:rho-logistic}
and therefore small perturbations of the curves are spotted and expected.

The left panel of Figure \ref{fig:rho-logistic} shows that the $\rho(\cdot)$ curves behave in a similar manner as the purely linear case, when one changes the price elasticity parameter $\alpha_0$:
a higher $\alpha_0$ leads to lower ``transition point'' but faster decay of revenue below it.
The middle panel of Figure \ref{fig:rho-logistic} compares the $\rho(\cdot)$ curves with different standard deviation of baseline utility distributions, while fixing other problem parameters.
Compared to the right panel of Figure \ref{fig:rho-linear} for the purely linear case, several difference exist. 
First, these curves no longer converge to the same point as $\delta_0$ increases, and the revenues of more variable distributions (larger standard deviation $\sigma$)
are considerably lower. Additionally, unlike the purely linear case when all curves are concave, for Logistic models when $\sigma$ is large the $\rho_0(\cdot)$ curve becomes less concave
and potentially non-concave eventually, shedding lights on the complexity of the problem with non-linear link functions. 

Finally, the right panel of Figure \ref{fig:rho-logistic} shows that, even if the underlying utility distribution shares the same mean and standard deviation, the actual shape of the distribution
is important and different distributions (with the same mean and standard deviation) give rise to different $\rho(\cdot)$ curves. Furthermore, the plot suggests for distributions with heavier tails
(Laplace, Student's t, etc.), the overall revenue curves are slightly lower than those associated with light-tailed distributions (uniform, Normal, etc.). 
This is an interesting finding and in stark contrast with the purely linear model, in which only the mean and the standard deviation of the underlying utility distribution would affect the revenue curve.

\section{Fair contextual pricing with incomplete information}

In this section, we present and analyze a dynamic contextual pricing algorithm (subject to approximate utility fairness constraints) when the demand model $\theta_0$ and $\alpha_0$
are \emph{unknown} in advance and must be learnt simultaneously with ongoing pricing decisions.
Following the set-up in the previous sections, we adopt the following generalized linear model 
\begin{equation}
\mathbb E[y|x,p] = f(x^\top\theta_0-\alpha_0 p),
\label{eq:defn-model-intercept}
\end{equation}
where $\theta_0\in\Theta\subseteq\mathbb R^d$, $\alpha_0\in A\subseteq\mathbb R_+$, with $\Theta,A$ being 
known classes that contain the unknown model parameters. 

To facilitate the design and analysis of bandit/learning algorithm under a utility fairness setting, we introduce the following additional technical assumption
on the properties of the distribution $P_{\mX}$ of the context vectors.
\begin{assumption}
There exists a constant $\sigma_x>0$ such that $\lambda_{\min}(\cov(P_{\mX}))\geq\sigma_x$,
where $\cov(P_{\mX})=\mathbb E_{P_{\mX}}[xx^\top]-(\mathbb E_{P_{\mX}}[x])(\mathbb E_{P_{\mX}}[x])^\top$ is the covariance matrix of $x\sim P_{\mX}$
and $\lambda_{\min}(\cdot)$ is the smallest eigenvalue of a positive semi-definite matrix.
\label{asmp:cov}
\end{assumption}

Intuitively, Assumption \ref{asmp:cov} states that the distribution of the context vectors $P_{\mX}$ is \emph{non-degenerate}, with its covariance matrix covering all directions in $\mathbb R^d$
and therefore equipped with a smallest eigenvalue bounded away from zero.
We adopt Assumption \ref{asmp:cov} for our bandit algorithm and design for two reasons.
First, it is a relatively intuitive and natural assumption and is satisfied whenever the context vectors are diverse, instead of the rarer situation in which context vectors are concentrated on
certain directions.
Furthermore, without Assumption \ref{asmp:cov}, it is not possible to estimate $\theta_0$ in $\ell_2$ norm (because some directions in $\mathbb R^d$ are not well covered by design),
and therefore it will be very challenging to design $\delta_0$-UF policies based on such deficient estimates because utility fairness is defined in a \emph{uniform} manner,
in which \emph{all} pairs of customers' offered prices must be similar simultaneously.

\subsection{Lower Bound}
\label{sec:lower}
We first establish a lower bound that demonstrate the fundamental difficulties of the fair contextual pricing problem with demand learning.
The lower bound is information theoretical, meaning that it applies to \emph{any} algorithm the implements fair pricing policies over $T$ time periods with policy adjustments
only made based on \emph{past} observations and data.

For $\delta_0>0$, $\theta_0\in\Theta$ and $\alpha_0\in A$, let $\mF_{\theta_0,\alpha_0}(\delta_0)$ denote the class of all contextual pricing policys $\pi:\mX\to[\underline p,\overline p]$ that satisfies
$\delta_0$-utility fairness ($\delta_0$-UF) as defined in Definition \ref{defn:fairness}. 
Our next theorem establishes an $\Omega(T^{2/3})$ lower bound on the optimal contextual pricing algorithm
over $T$ consecutive time periods, provided that such an algorithm uses $\delta_0$-UF policies with high probability.
Note that in the lower bound the constructed problem instances satisfy \emph{all} assumptions made up to this point,
including the covariance non-degeneracy assumption \ref{asmp:cov}.
\begin{theorem}[$T^{2/3}$ regret lower bound]
Fix arbitrary $\delta_0\in(0,1/2]$ and let $\underline C>0$ be a strictly positive constant potentially depending on $\delta_0$.
Let $f(u)\equiv u$ be the purely linear demand model.
There exists model set-up $\mX$, $P_{\mX}$, $[\underline p,\overline p]$ and function classes $\Theta,A$ satisfying Assumptions \ref{asmp:boundedness}-\ref{asmp:cov} such that, for any contextual pricing with demand learning algorithm over $T$ time periods implementing pricing policies $\pi_1(\cdot),\cdots,\pi_T(\cdot)$, if such an algorithm satisfies
$$
\Pr[\pi_1,\cdots,\pi_T\in\mF_{\theta_0,\alpha_0}(\delta_0)] \geq 0.95, \;\;\;\;\;\;\forall\theta_0,\alpha_0\in\Theta\times A,
$$
then
$$
\sup_{\theta_0,\alpha_0\in\Theta\times A}\mathbb E_{\theta_0,\alpha_0}\left[\sum_{t=1}^T \max_{\pi\in\mF_{\theta_0,\alpha_0}(\delta_0)}R_{\theta_0,\alpha_0}(\pi)-R_{\theta_0,\alpha_0}(\pi_t)\right]\geq \underline C\times T^{2/3},
$$
where $R_{\theta,\alpha}(\pi) = \mathbb E[\pi(x)(x^\top\theta-\alpha\pi(x))|x\sim P_{\mX}]$ is the expected revenue of policy $\pi$.
\label{thm:lower-bound}
\end{theorem}

The $\Omega(T^{2/3})$ regret lower bound established in Theorem \ref{thm:lower-bound} is quite surprising, because virtually all contextual bandit or parametric bandit pricing problems
admit regret upper bounds of $\tilde O(\sqrt{T})$ or $O(\log T)$ \citep{rusmevichientong2010linearly,li2011pricing}. 
To understand why the fair contextual pricing problem with demand learning admits an $\Omega(T^{2/3})$ regret lower bound,
it is instructive to understand the structure of the optimal $\delta_0$-UF policy (for $\delta_0$ being not too large) specified in Theorem \ref{thm:optimal-pi}:
the optimal price of an incoming user with context $x$ is a linear function of $x$, which means that there is \emph{perfect co-linearity} between the context vectors
and offered prices, if the pricing strategy is optimal.
This creates a dilemma: if near-optimal pricing policy is implemented, then the high co-linearity between $x$ and $p$ means that we cannot estimate $\theta_0$ and $\alpha_0$
both accurately at the same time; on the other hand if prices are given to circumvent such co-linearity, such a pricing strategy must be far from optimal because it violates the structure
derived in Theorem \ref{thm:optimal-pi}.
In the supplementary material, we make the above argument formal by explicitly constructing adversarial problem instances and then using Pinsker's inequality together with other
information theory tools to rigorously prove the $\Omega(T^{2/3})$ regret lower bound in Theorem \ref{thm:lower-bound}.

\subsection{Algorithm design}
\label{sec:algo}
\begin{algorithm}[t]
\caption{Contextual pricing with utility fairness constraints and demand learning}
\label{alg:bandit-uf}
\begin{algorithmic}[1]
\State \textbf{Input}: likelihood function $\mL$, price domain $[\underline p,\overline p]$, utility fairness parameter $\delta_0>0$, time horizon $T$, algorithm parameters $\kappa_1,\kappa_2,K>0$.
\State For $T_0=\lceil T^{2/3}\rceil$ time periods, offer $\underline p$ or $\overline p$ with half probabilities each;  \Comment{{\color{blue}price experimentation}}
\State  $\tilde\delta_0=\max\{0,\delta_0-\kappa_1/\sqrt{T_0}\}$; \Comment{{\color{blue} cushion of fairness for estimation error}}
\State $\hat\theta,\hat\alpha\gets\arg\max_{\theta,\gamma}\sum_{t=1}^{T_0} \ln\mL(y_t|x_t^\top\theta-0.5\gamma p_t)$; \Comment{{\color{blue}maximum likelihood estimation}}\label{line:mle}
\State $\underline\pi\gets \underline p-\tilde\delta_0\max_{t\leq T_0}x_t^\top\hat\theta$, $\overline\pi\gets \bar p-\tilde\delta_0\min_{t\leq T_0}x_t^\top\hat\theta$;
\State Discretize $[\underline \pi,\overline \pi]$ into $K$ evenly spaced prices $\{\pi_k\}_{k=1}^K$ and $r_k\gets n_k\gets 0$;\Comment{{\color{blue}initialization of MAB}}
\For{$t=T_0+1,\cdots,T$}
	\State $k_t\gets\arg\max_{k\in[K]}r_k/n_k+\kappa_2/\sqrt{n_k}$\textsuperscript{$\dagger$}; \Comment{{\color{blue} UCB construction}}
	\State Implement policy $\pi_t(x)=\trim_{[\underline p,\overline p]}(\pi_{k_t}+\tilde\delta_0x^\top\hat\theta)$\textsuperscript{$*$};
	\State $r_{k_t}\gets r_{k_t}+y_t\pi_t(x_t)$, $n_{k_t}\gets n_{k_t}+1$; \Comment{{\color{blue}update revenue counts}}
\EndFor
\end{algorithmic}
{\footnotesize \textsuperscript{$*$} $\trim[a,b](u)$ takes the value of if $u<a$, $b$ if $u>b$, and $u$ if $a\leq u\leq b$;}\\
{\footnotesize \textsuperscript{$\dagger$} If for some $k$, $n_k=0$, select $k_t$ as the smallest such $k$.}
\end{algorithm}

Algorithm \ref{alg:bandit-uf} gives a pseudo-code description of our proposed contextual pricing algorithm with demand learning, subject to utility fairness constraints.
At a higher level, Algorithm \ref{alg:bandit-uf} can be divided into two phases.
The first phase consists of $T_0=\lceil T^{2/3}\rceil$ time periods, during which price experimentation is carried out by offering non-personalized fixed price to all incoming customers,
at one lower price $p_L$ and one higher price $p_U$. In the algorithm the two prices are set as $p_L=\underline p$ and $p_U=\overline p$,
but they can be changed to any pair of prices in $[\underline p,\overline p]$ that are not too close to each other.
After the price experimentation phase ends and observations are collected, a \emph{maximum likelihood estimation} (MLE) procedure is carried out to obtain estimates $\hat\theta,\hat\alpha$
of the unknown model parameters $\theta_0,\alpha_0$. The MLE formulation involves a likelihood function $\mL$ as an input to the algorithm, which should be designed for different link functions
$f$. More details and examples of such a likelihood function will be given later.

The second phase of Algorithm \ref{alg:bandit-uf} is an application of the upper-confidence bound (UCB) algorithm for multi-armed bandit (MAB), with arms being $K$ discretized
``starting prices'' in $[\underline p,\overline p]$. Each starting price $\pi_k$, $k=1,2,\cdots,K$ completely characterizes a contextual pricing policy $\pi(x)=\pi_k+\tilde\delta_0 x^\top\hat\theta$
which will be proved to satisfy $\delta_0$-utility fairness with high probability. The number of discretized prices $K$ will scale as a polynomial function of the time horizon $T$,
which is specified in our Theorem \ref{thm:upper-bound} later.
While such an approach does \emph{not} usually lead to an $\tilde O(\sqrt{T})$ regret upper bound, we note that since our regret target is $\tilde O(T^{2/3})$ (thanks to Theorem \ref{thm:lower-bound})
it is sufficient for our purpose.

Algorithm \ref{alg:bandit-uf} requires a ``likelihood function" $\mL$ as input, which is designed according to the link function $f$.
In general, this likelihood function coincides with the common understanding of likelihood functions in statistics, as shown by several examples later.
To be fully general, we consider any $\mL$ function as a valid likelihood or risk function as long as it satisfies Assumption \ref{asmp:likelihood} stated below.
To simplify the notations, we use $\beta_0=(\theta_0,\alpha_0)\in\mathbb R^{d+1}$ as the extended model vector.
\begin{assumption}[Likelihood function]
$\ln\mL$ is twice continuously differentiable with respect to $\beta=(\theta,\alpha)$ on $\mathbb R^{d+1}$.
Furthermore, there exist constants $\rho_L>0$ and $0<\sigma_L\leq M_L<\infty$ such that for every $p\in[\underline p,\overline p]$ and $z=(x,-p)$, 
the following hold:
\begin{enumerate}
\item $\mathbb E_{z,\beta_0}[\nabla_\beta\ln\mL(y|z,\beta_0)]=0$;
\item For every $y\in\{0,1\}$ and $\beta\in\mathbb R^{d+1}$, $\nabla_{\beta\beta}^2\ln\mL(y|z|\beta)\preceq 0$;
\item For every $y\in\{0,1\}$ and $\beta\in\mathbb R^{d+1}$ such that $\|\beta-\beta_0\|_2\leq\rho_L$, $\|\nabla_\beta\ln\mL(y|z,\beta_0)\|_2\leq M_L$ and $-\nabla_{\beta\beta}^2\ln\mL(y|z,\beta_0)\succeq \sigma_L zz^\top$.
\end{enumerate}
\label{asmp:likelihood}
\end{assumption}

The first property of Assumption \ref{asmp:likelihood} states that the true underlying parameter $\beta_0$ is the stationary point of $\ln\mL$ in expectation;
the second property shows that $\ln\mL$ is concave in $\beta$;
the third property shows that if $\beta$ is in a neighborhood of the true parameter $\beta_0$, the gradient of $\mL$ is bounded almost surely and furthermore $\ln\mL$ is locally strongly concave.
When $\mL$ is the likelihood of a generative model, $\nabla_\beta\ln\mL$ corresponds to the \emph{score function} and $-\nabla_{\beta\beta}^2\ln\mL$ corresponds to the
\emph{Fisher's information}, and all properties in Assumption \ref{asmp:likelihood} are standard in statistics for regular generalized linear models.

\begin{example}[Linear demand]
For linear demand $f(u)=u$, the likelihood function can be chosen as $\mL(y|z,\beta_0)=e^{-(y-z^\top\beta_0)^2/2}$, which satisfies Assumption \ref{asmp:likelihood} with parameters 
$\rho_L=B$, $M_L=\diam(\mX)$ and $\sigma_L=1$, where $\diam(\mX)=\sup_{x,x'\in\mX}\|x-x'\|_2$ is the diameter of $\mX$.
\end{example}

\begin{example}[Logistic demand]
For Logistic demand $f(u)=e^u/(1+e^u)$, the likelihood function can be chosen as $\mL(y|z,\beta_0)=ye^{z^\top\beta_0}/(1+e^{z^\top\beta_0})+(1-y)/(1+e^{z^\top\beta_0})$,
which satisfies Assumption \ref{asmp:likelihood} with parameters $\rho_L>0$, $M_L=\diam(\mX)$ and $\sigma_L=0.5f(B+\diam(\mX)\rho_L)$. 
\end{example}

\begin{example}[Exponential demand]
For exponential demand $f(u)=1-e^{-u}$, the likelihood function can be chosen as $\mL(y|z,\beta_0)=y(1-e^{-z^\top\beta_0})+(1-y)e^{-z^\top\beta_0}$,
which satisfies Assumption \ref{asmp:likelihood} with parameters $\rho_L>0$, $M_L=(e^{-u_{\min}}/(1+e^{-u_{\min}}))\diam(\mX)$ and $\sigma_L=e^{-B-\diam(\mX)\rho_L}$,
where $u_{\min}=\inf_{z}z^\top\beta_0>0$.
\end{example}

With the second property of the likelihood function $\mL$ in Assumption \ref{asmp:likelihood}, Line \ref{line:mle} is a convex minimization problem and could be solved by
any standard convex optimization packages for low and intermediate dimensional problems.

\subsection{Fairness and regret analysis}

	In this section we present the main theorem analyzing the fairness and revenue performance of Algorithm \ref{alg:bandit-uf}.

	\begin{theorem}
	Suppose Assumptions \ref{asmp:boundedness}-\ref{asmp:likelihood} hold, and $\delta_0\leq\sigma_u/M_r$ so that the optimal (full-information) $\delta_0$-UF policy
	admits the linear form in Eq.~(\ref{eq:optimal-linear-form}) as proved in Theorem \ref{thm:optimal-pi}. 
	Suppose also that Algorithm \ref{alg:bandit-uf} is run with parameters $\kappa_1=\frac{8M_r\diam(\mX)\sqrt{\ln(dT)}}{\min\{\sigma_x,0.25(\overline p-\underline p)^2\}\sigma_r}$, $\kappa_2=4\sqrt{\ln T}$ and $K=\lceil T^{1/3}\rceil$, where $\diam(\mX) = \sup_{x,x'\in\mX}\|x-x'\|_2$ is the diameter of $\mX$.
	Let $\pi_1,\cdots,\pi_T$ be the contextual pricing policies implemented during the $T$ time periods.
	With probability $1-\tilde O(T^{-2})$, the following hold:
	\begin{enumerate}
	\item {\textbf{(Fairness).}} All $\pi_1,\cdots,\pi_T$ satisfy $\delta_0$-utility fairness as defined in Definition \ref{defn:fairness} with respect to the 
	underlying model $\theta_0$ and $\alpha_0$;
	\item \textbf{(Regret).} $\sum_{t=1}^T\mathbb E_{x\sim P_{\mX}}[r_x(\pi^*(x))-r_x(\pi_t(x))]\leq (6L_f\tilde B\kappa_1 + 4L_f\kappa_2) T^{2/3}$, where $\pi^*$ is the optimal $\delta_0$-UF policy and 
	$r_x(p)=pf(x^\top\theta_0-\alpha_0p)$ is the expected revenue at $x\in\mX$ with offered price $p$.
	\end{enumerate}
	\label{thm:upper-bound}
	\end{theorem}
	
	Theorem \ref{thm:upper-bound} contains two results. The first result is the \emph{fairness} guarantee, which shows that with high probability \emph{all} contextual pricing policies 
	$\pi_1,\cdots,\pi_T$ implemented by Algorithm \ref{alg:bandit-uf} over the entire $T$ time periods satisfy $\delta_0$-utility fairness.
	The failure probability $\tilde O(T^{-2})$ is particularly small, meaning that it is unlikely to seen even one failure event over all $T$ time periods.
	Therefore, when implemented the firm can be very confident that all pricing decisions made satisfy required fairness constraints.
	
	The second result of Theorem \ref{thm:upper-bound} analyzes the \emph{revenue} performance, by upper bounding the cumulative regret between Algorithm \ref{alg:bandit-uf} 
	and a benchmark optimal policy with full information. In general, the cumulative regret of Algorithm \ref{alg:bandit-uf} is asymptotically on the order of $\tilde O(T^{2/3})$,
	which is sub-linear in time horizon $T$ demonstrating meaningful learning, while at the same time also matches the information-theoretical lower bound
	established in Theorem \ref{thm:lower-bound}. 
	
\subsection{Computational results}

\begin{figure}[t]
\centering
\includegraphics[width=0.45\textwidth]{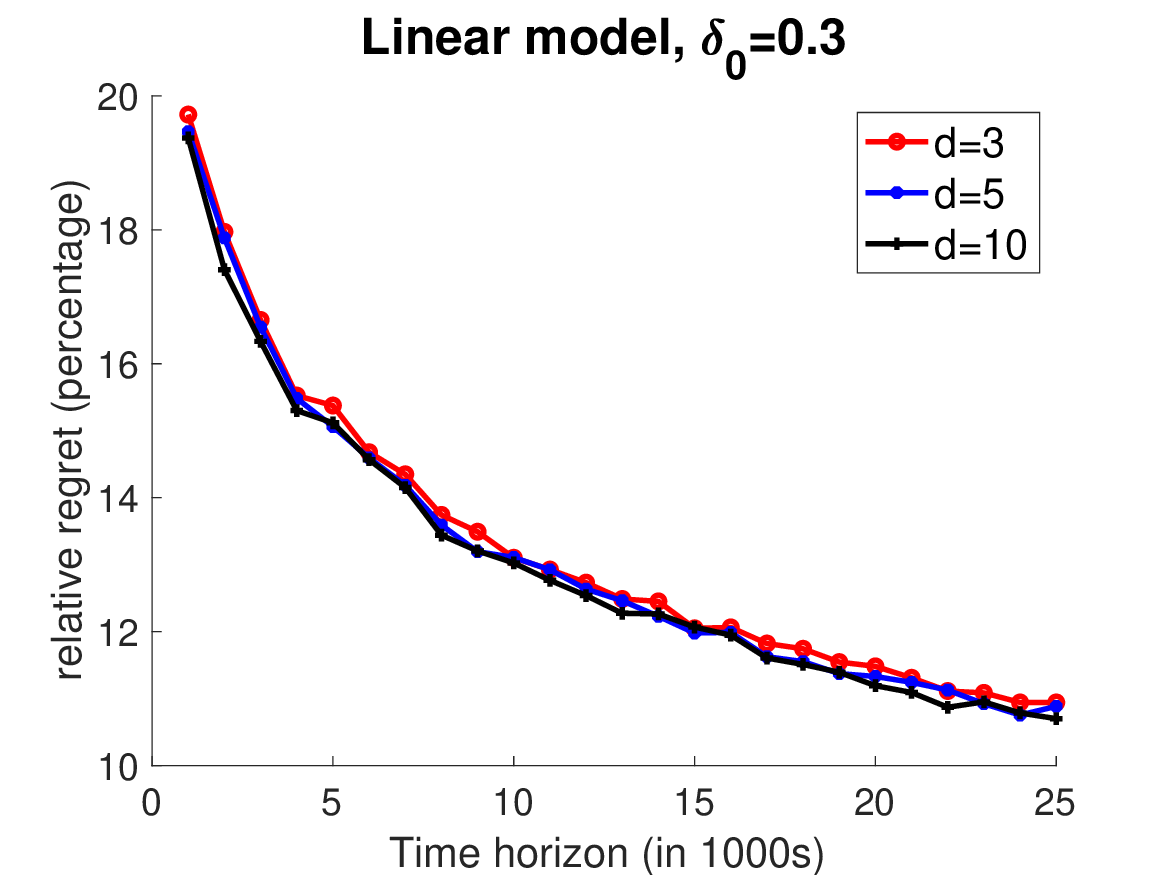}
\includegraphics[width=0.45\textwidth]{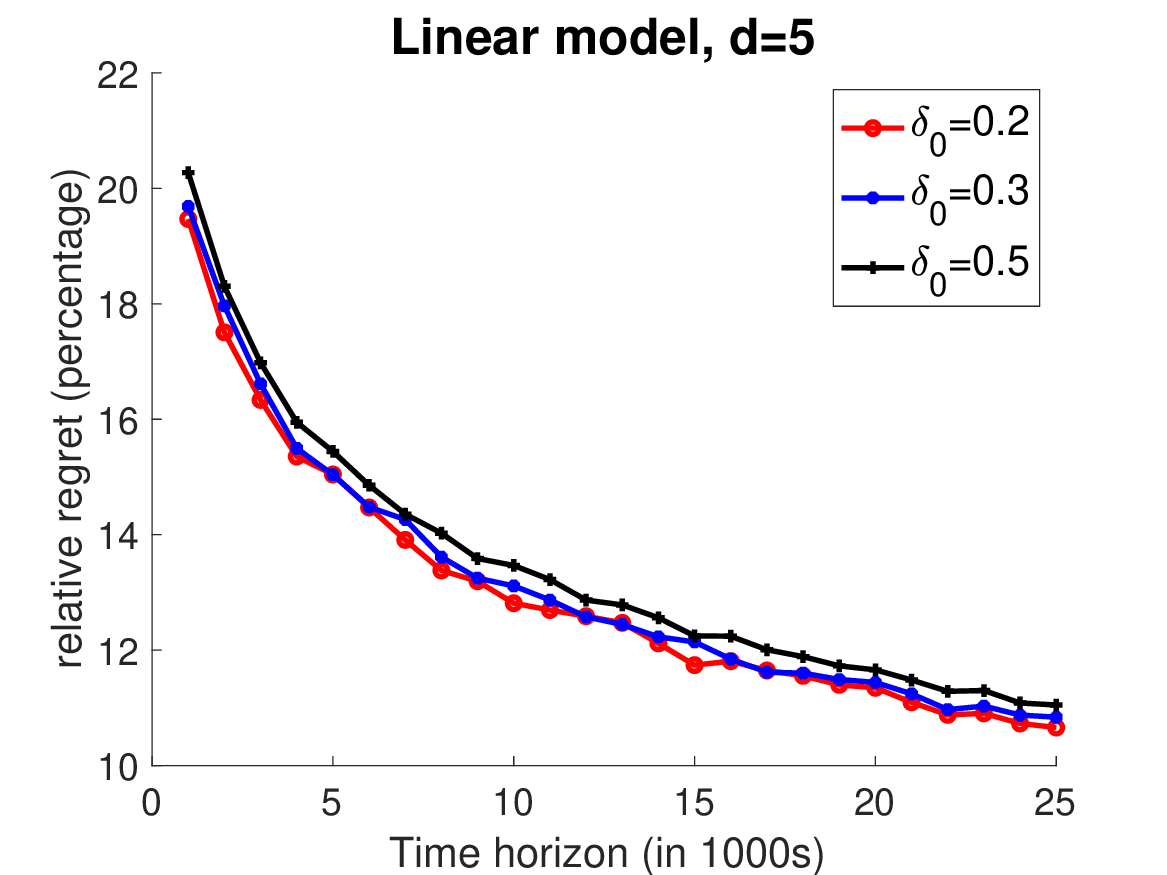}
\includegraphics[width=0.45\textwidth]{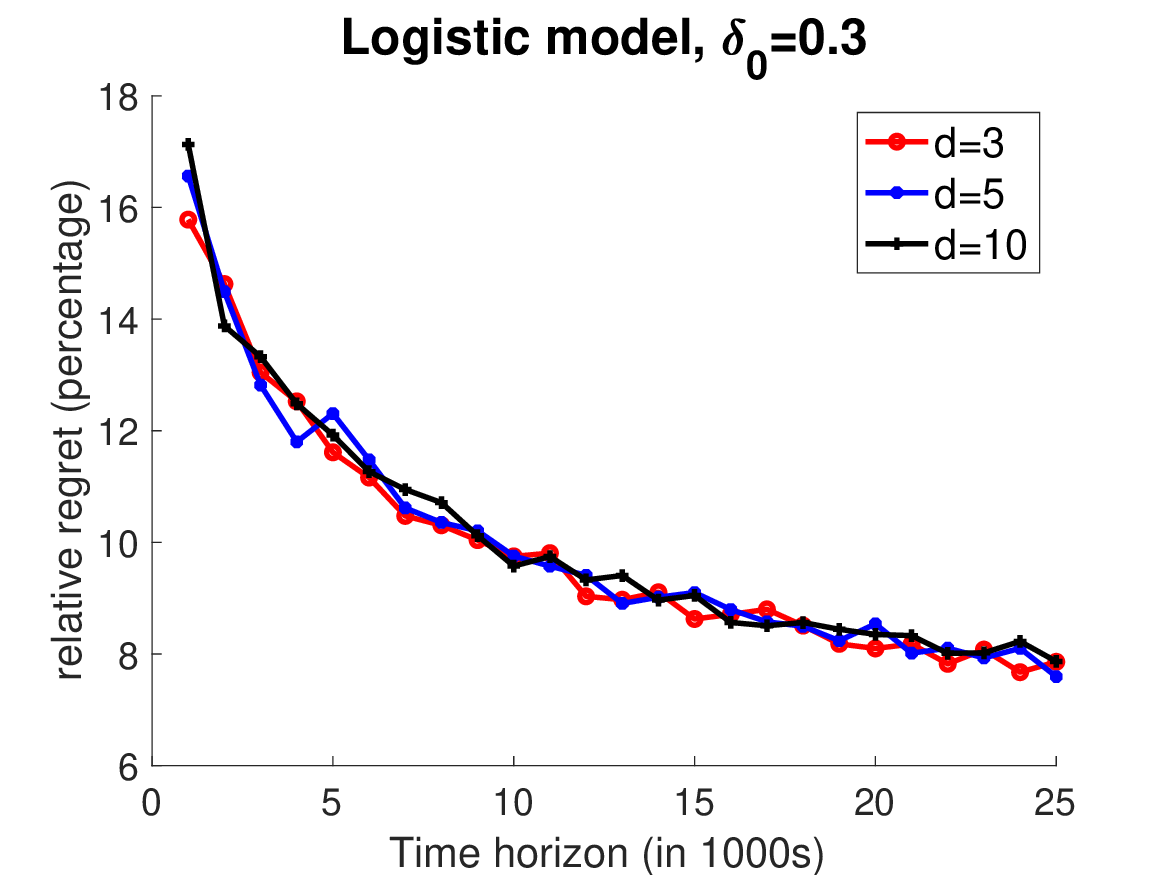}
\includegraphics[width=0.45\textwidth]{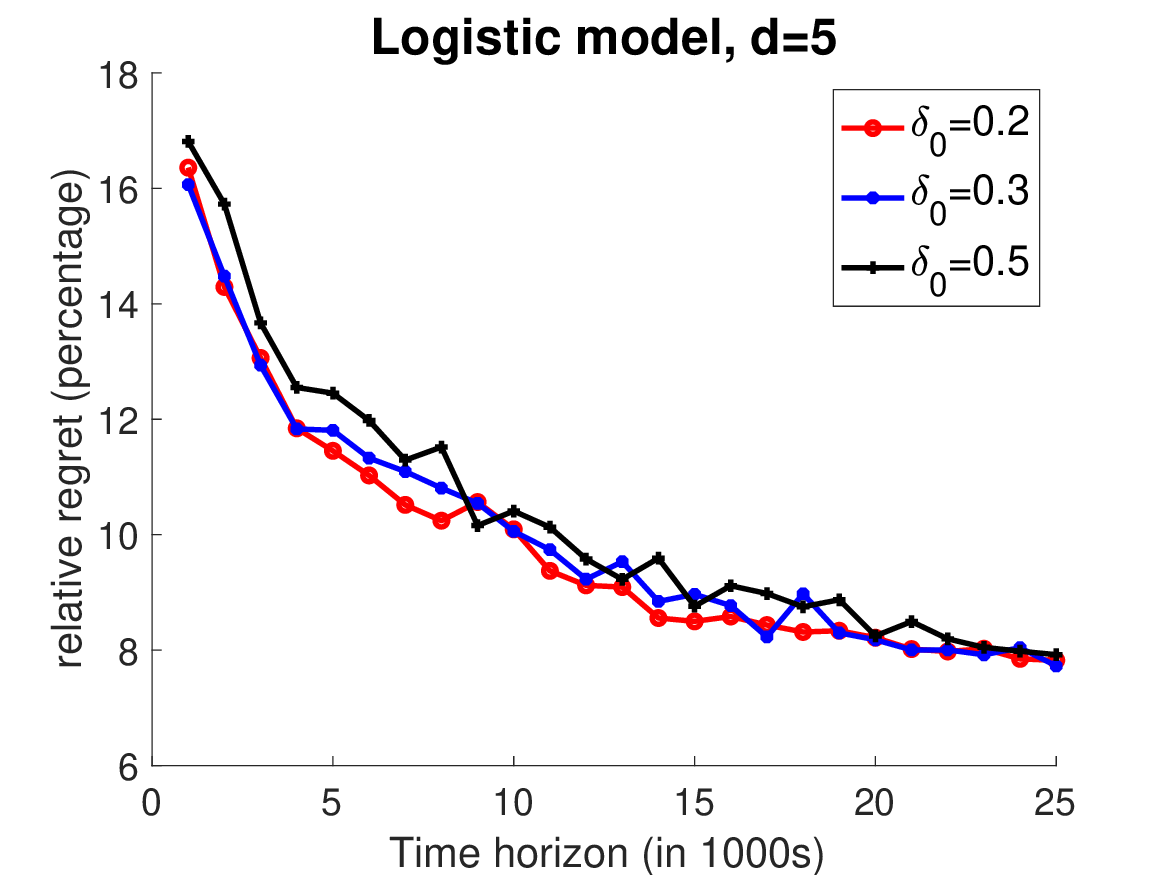}
\caption{Plots of relative regret versus time horizon $T$}
{\small Note: the first row is for purely linear demand and the second row is for Logistic demand. The left panel is for different context dimensionality $d$ and the right panel is for different utility
fairness parameters $\delta_0$. The time horizon $T$ ranges from 1000 to 25000. All experimental settings are repeated for 20 independent trials and the mean relative regret is reported.}
\label{fig:plot-bandit-main}
\end{figure}	

\begin{figure}[t]
\centering
\includegraphics[width=0.6\textwidth]{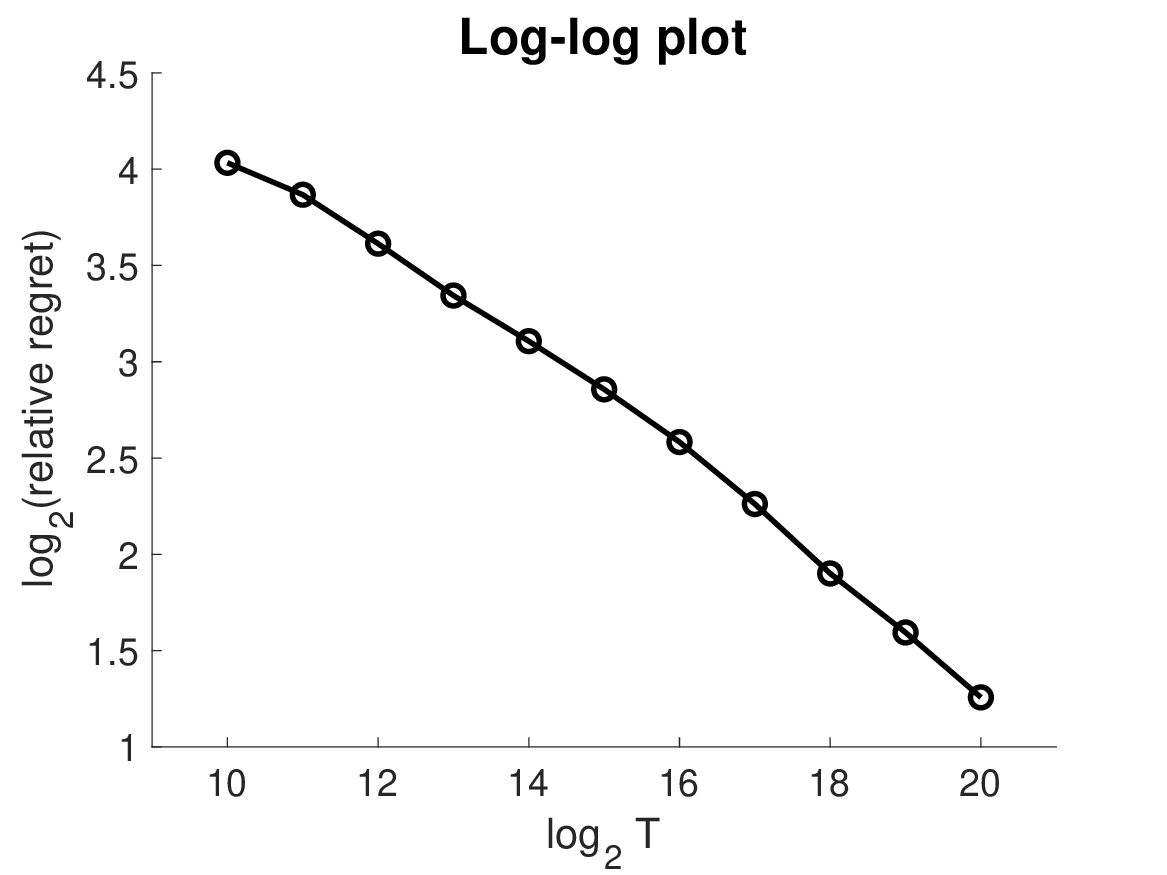}
\caption{Log-Log plot of relative regret versus time horizon $T$.}
{\small Note: Logistic demand model, context dimension $d=3$, utility fairness parameter $\delta_0=0.3$. The $y$-axis is $\log_2$-relative regret and
the $x$-axis is $\log_2 T$, where the time horizon $T$ ranges from $2^{10}$ to $2^{20}$. All experimental settings are repeated for 20 independent trials and the mean relative regret is reported.}
\label{fig:plot-bandit-loglog}
\end{figure}

We complement our theoretical findings with some computational results.
Algorithm \ref{alg:bandit-uf} is implemented in Julia, with plots made by Matlab.
The two algorithmic parameters $\kappa_1,\kappa_2,K$ are simply selected as $\kappa_1=\sqrt{\ln(dT)}$, $\kappa_2=\sqrt{\ln T}$ and $K=\lceil T^{1/3}\rceil$
which are on the same asymptotic order of the theoretical values of these parameters.
The linear and Logistic demand functions are studied, with likelihood functions $\mL$ selected as specified in examples in the previous section.

In figure \ref{fig:plot-bandit-main}, \emph{relative regret} (the difference between the cumulative revenue of the optimal $\delta_0$-UF policy and the bandit algorithm, normalized
by the cumulative revenue of the optimal $\delta_0$-UF policy itself) is plotted under various experimental settings.
It can be seen that, under all settings, the relative regret converges to zero as the number of time horizon $T$ increases, corroborating the theoretical sub-linear regret results established in the previous section.

We further produce a log-log plot between the relative regret and the time horizon in Figure \ref{fig:plot-bandit-loglog}, to investigate the asymptotic relationship between the relative regret and $T$
when $T$ is large. The observations lie roughly on a straight line in the log-log plot, suggesting that the cumulative regret grows as a sub-linear, polynomial function of $T$.
Furthermore, simple linear regression estimates a slope of -0.28 based on the observations produced in Figure \ref{fig:plot-bandit-loglog}, which implies an $O(T^{0.72})$ cumulative regret scaling
and is close to the $\tilde O(T^{2/3})$ theoretical scaling proved in Theorem \ref{thm:upper-bound}.

\section{Conclusions}

	In this paper, we have developed contextual bandit algorithm tailored for personalized pricing under the  constraints of utility fairness and uncertain demand, achieving an optimal regret upper bound. Our investigation reveals the nuanced trade-offs between fairness, utility, and revenue maximization, significantly contributing to the discourse in both dynamic pricing and fairness-ware machine learning fields.Several intriguing avenues for future research emerge from our study. Firstly, we aim to extend our exploration into dynamic pricing settings with more general demand functions, such as semi-parametric or non-parametric demands. Additionally, examining the implications of various fairness constraints within the context of regulatory compliance presents an interesting challenge. Finally, conducting a longitudinal study to assess the long-term impacts of fairness-constrained pricing strategies on customer satisfaction, loyalty, and corporate reputation would provide valuable insights. Through these future research directions, we aim to deepen the understanding of fairness in dynamic pricing, fostering a more ethical and equitable landscape in algorithmic decision-making.
	
\section*{Acknowledgment} 

Xi Chen would like to acknowledge the support from the NSF via Grant IIS-1845444.

	\bibliography{refs}
	\bibliographystyle{apa-good}
	\newpage
	
	\ECSwitch
	
	
	\newcommand{\mD}{\mathcal D}
	
	
	
	\ECHead{Appendix: additional proofs}
	
	\section{Omitted proofs of technical lemmas in Sec.~\ref{sec:fullinfo}}
	
	\subsection{Proof of Lemma \ref{prop:approx}}
First, from the definition of $\iota_k(\cdot)$ and the fact that $|j_{k+1}^*-j_k^*|\leq 1$ for all $k$, the defined $\varpi_\epsilon^*(\cdot)$ is $\delta_0$-Lipschitz continuous.
This immediately implies that $\pi_\epsilon^*$ satisfies $\delta_0$-UF.

To prove the lower bound on the expected revenue of $\pi_\epsilon^*$, we shall construct a sequence $j_1^\circ,\cdots,j_{M_n}^\circ$ feasible to Eq.~(\ref{eq:inf-opt-ver3})
that is approximately as profitable as the non-discretized optimal policy $\pi^*$.
Recall that, thanks to Eq.~(\ref{eq:inf-opt-ver2}), the optimal $\pi^*$ can be parameterized as $\pi^*(x)=\varpi^*(x^\top\theta_0)$.
Let $p_1^* = \varpi^*(u_1)$, $p_2^*=\varpi^*(u_2)$, and so on.
Let $j_1^\circ=\arg\min_{j\in[M_p]}|p_j-p_1^*|$, with ties broken arbitrarily. For $k=2,3,\cdots,M_u$, determine the values of $j_k^\circ$ in the following ways:
\begin{enumerate}
\item If $p_{j_{k-1}^\circ}\leq p_{k-1}^*\leq p_k^*$ then set $j_k^\circ = j_{k-1}^\circ + 1$;
\item If $p_{j_{k-1}^\circ}\geq p_{k-1}^*\geq p_k^*$ then set $j_k^{\circ}=j_{k-1}^\circ-1$;
\item In all other scenarios, set $j_{k}^\circ=j_{k-1}^\circ$.
\end{enumerate}
Because $p_{j+1}=p_j+\delta_0\epsilon$ for every $j$ and $|p_{k+1}^*-p_k^*|\leq \delta_0|u_{k+1}-u_k|\leq \delta_0\epsilon$ for every $k$, thanks to the fact that $\pi^*$ satisfies $\delta_0$-UF,
it is easy to verify via induction that that $\{j_k^\circ\}_{k=1}^{M_u}$ sequence constructed above satisfies $|p_{j_k^\circ}-p_k^*|\leq \delta_0\epsilon$.
Construct a policy $\pi_\epsilon^\circ(x)=\varpi_\epsilon^\circ(x^\top\theta_0)$ with
	\begin{align}
	\varpi_\epsilon^\circ(u) = p_{j_{k(u)}^\circ} + \delta_0\iota_{k(u)}(u-u_{k(x)};\{j_k^\circ\}) \;\;\;\;\;\;\text{where}\;\; k(u) = \arg\min_{k\in[M_u]} \big|u_k-u\big|,
	\label{eq:picirc-epsilon}
	\end{align}
	and the $\iota_k(\cdot;\{j_k\}_k)$ function being defined in Eq.~(\ref{eq:defn-iota}). This definition together with the properties that $|p_{j_k^\circ}-p_k^*|\leq \delta_0\epsilon$ for all $k$
	and that $\pi^*$ is $\delta_0$-UF, yields
	\begin{equation}
	\sup_{u} \big|\varpi_\epsilon^\circ(u)-\varpi^*(u)\big|\leq 2\delta_0\epsilon.
	\label{eq:proof-propapprox-1}
	\end{equation}
	By Assumptions \ref{asmp:boundedness} and \ref{asmp:holder-sc}, the expected revenue function $r_u(\cdot)$ is $2L_f$-Lipschitz continuous. Subsequently,
	\begin{align*}
	\mathbb E_{x\sim P_{\mX}}&\left[r_{x^\top\theta_0}(\pi^*(x))-r_{x^\top\theta_0}(\pi_\epsilon^\circ(x))\right]
	= \int_{-B}^B \left(r_u(\varpi^*(u))-r_u(\varpi_\epsilon^\circ(u))\right)\ud P_{\mU}(u)\\
	&\leq \int_{-B}^B 2L_f\big|\varpi^*(u)-\varpi_\epsilon^\circ(u)\big|\ud P_{\mU}(u)\leq 4L_f\delta_0\epsilon,
	\end{align*}
	which is to be proved. $\square$
	
	\subsection{Proof of Lemma \ref{lem:approx-uniform}}
	For notational simplicity define 
	$$
	\varphi(\varpi) := \mathbb E_{P_{\mU}}[r_u(\varpi(u)]
	$$
	for every $\varpi\in\mF$. Because $r_u(\cdot)$ is $2L_f$-Lipschitz continuous for all $u$, it is easy to verify that, for every pair of $\varpi,\tilde\varpi\in\mF$, 
	$$
	\big|\varphi(\varpi)-\varphi(\tilde\varpi)\big| \leq \int_{-B}^B 2L_f|\varpi(u)-\tilde\varpi(u)|\ud P_{\mU}(u) \leq 2L_f\|\varpi-\tilde\varpi\|_{\infty},
	$$
	meaning that $\varphi(\cdot)$ is a $2L_f$-Lipschitz continuous function in $\|\cdot\|_{\infty}$.
	
	Fix arbitrary $\epsilon>0$ and consider
	$$
	\mF_{\epsilon} := \{g\in\mF: \|g-\varpi^*\|_{\infty}\geq \epsilon\}.
	$$
	It is easy to verify that $\mF_\epsilon$ is closed and bounded in $\|\cdot\|_{\infty}$. Furthermore, since $\mF_\epsilon\subseteq\mF$ and $\mF$ only contains $\delta_0$-Lipschitz continuous functions,
	any function sequence in $\mF_\epsilon$ is equi-continuous. By the Arzela-Ascoli theorem, this means that $\mF_\epsilon$ is compact with respect to $\|\cdot\|_{\infty}$.
	Since $\varphi$ itself is Lipschitz continuous, this means that the image $\varphi(\mF_\epsilon)=\{\varphi(\varpi):\varpi\in\mF_\epsilon\}$ is compact and therefore
	$$
	\beta := \varphi(\varpi^*)-\sup_{\varpi\in\mF_\epsilon}\varphi(\varpi) = \varphi(\varpi^*)-\max_{\varpi\in\mF_\epsilon}\varphi(\varpi) > 0,
	$$
	because the optimal policy $\pi^*$ is unique and $\pi^*\notin\mF_\epsilon$.
	Because $\lim_{n\to\infty}\varphi(\varpi_n)=\varphi(\varpi^*)$, there exists $N\in\mathbb N$ such that for all $n\geq N$, $\varphi(\varpi^*)-\varphi(\varpi_n)\leq\beta/2$.
	This implies that $\varpi_n\notin\mF_\epsilon$ for all $n\geq N$, meaning that $\|\varpi_n-\varpi^*\|_{\infty}\leq\epsilon$ for all $n\geq N$.
	The lemma is thus proved because $\epsilon>0$ is arbitrary. $\square$

	\subsection{Proof of Lemma \ref{lem:optimal-obj-continuity}}
	First, note that $R(\cdot)$ is a monotonically decreasing function because a policy that satisfies $\delta$-UF also satisfies $\delta'$-UF for any $\delta'>\delta$.
	Let $\pi^*(x)=\varpi^*(x^\top\theta_0)$ be the $\delta_0$-UF policy that maximizes its expected revenue, which exists because the expected revenue 
	is continuous in $\|\cdot\|_{\infty}$, and the set of $\delta_0$-UF policies
	coincide with the set of $\delta_0$-Lipschitz continuous functions on $[-B,B]$, which is compact in $\|\cdot\|_{\infty}$ thanks to the Arzela-Ascoli theorem.
	For $\epsilon>0$ being arbitrarily small,
	consider the policy $\tilde\varpi(u) = \max\{\underline p, (1-\epsilon/\delta_0)\varpi^*(u)\}$.
	It is easy to verify that $\tilde\varpi(\cdot)$ satisfies $(\delta_0-\epsilon)$-UF.
	On the other hand, the expected revenue of $\tilde\varpi$ and $\varpi^*$ differs only by $O(\epsilon)$ because the expected revenue function $r_u(\cdot)$
	is $2L_f$-Lipschitz continuous, for all $u$. Taking the limit of $\epsilon\to 0^+$ we have proved Lemma \ref{lem:optimal-obj-continuity}. $\square$
	
	\subsection{Proof of Lemma \ref{lem:pstar-mono}}
For every $p\in[\underline p,\overline p]$ and $u\in[-B,B]$, define $r(u,p):=pf(u-\alpha_0 p)$.
Because $p^*(u),p^*(u')$ are interior minimizers, the first-order KKT condition asserts that
\begin{align}
\partial_p r(u,p^*(u))=\partial_p r(u',p^*(u'))=0.\label{eq:proof-pstar-mono-1}
\end{align}
Additionally, using the chain rule and Eq.~(\ref{eq:defn-sigma-u}), we have for every $u\in[-B,B]$ and $p\in[\underline p,\overline p]$ that
\begin{align}
\partial_u\partial_p r(u,p) &= \partial_u\left(f(u-\alpha_0 p)-\frac{\alpha_0 p}{2}f'(u-\alpha_0 p)\right) =f'(u-\alpha_0 p)-\frac{\alpha_0p}{2}f''(u-\alpha_0 p)\geq\sigma_u.
\label{eq:proof-pstar-mono-2}
\end{align}
Incorporating Eq.~(\ref{eq:proof-pstar-mono-2}) into Eq.~(\ref{eq:proof-pstar-mono-1}) and using the mean-value theorem, there exists $\tilde u\in[u,u']$ such that
\begin{align}
\partial_p r(u',p^*(u)) &= \partial_p r(u,p^*(u)) + \partial_u\partial_p r(\tilde u,p^*(u))(u'-u)\geq \sigma_u(u'-u) > 0.
\label{eq:proof-pstar-mono-3}
\end{align}
Because $r(u',\cdot)$ is unimodal and peaks at $p^*(u')$, the fact that $\partial_p r(u',p^*(u))>0$ implies that $p^*(u)<p^*(u')$. Furthermore, because $r(u',\cdot)$ is twice continuously differentiable
with its second-order derivatives being uniformly upper bounded by $M_r$, it holds that 
\begin{align}
\big|\partial_p r(u',p^*(u))\big| \leq M_r \big|p^*(u)-p^*(u')\big|.\label{eq:proof-pstar-mono-4}
\end{align}
Combining Eqs.~(\ref{eq:proof-pstar-mono-3},\ref{eq:proof-pstar-mono-4}) we obtain a lower bound on $p^*(u')-p^*(u)$, which is the first inequality in Lemma \ref{lem:pstar-mono}.

To prove the second inequality, note that $r(u,p)$ is $L_f$-Lipschitz continuous in $u$ for every $p$ because $f(\cdot)$ is $L_f$-Lipschitz continuous thanks to Assumption \ref{asmp:holder-sc}.
This means that
\begin{align}
r(u',p^*(u'))\leq r(u,p^*(u')) + L_f(u'-u) \leq r(u,p^*(u))+L_f(u'-u) \leq r(u',p^*(u))+2L_f(u'-u).
\label{eq:proof-pstar-mono-4}
\end{align} 
On the other hand, the strong uni-modality property of $r_u$ stated in Assumption \ref{asmp:holder-sc} implies that
\begin{align}
r(u',p^*(u))\leq r(u',p^*(u')) - \frac{\sigma_r}{2}(p^*(u')-p^*(u))^2.\label{eq:proof-pstar-mono-5}
\end{align}
Combining Eqs.~(\ref{eq:proof-pstar-mono-4},\ref{eq:proof-pstar-mono-5}) we obtain the second inequality in Lemma \ref{lem:pstar-mono}.
$\square$

\section{Proof of Theorem \ref{thm:optimal-pi}}

To prove this theorem we rely again on a discretization idea similar to the one developed in the previous section, but for mathematical proof purposes instead of efficient computation.
Let $\epsilon>0$ be a small positive discretization error parameter and $\{u_k\}_{k=1}^{M_u}$ be the discretized grids of $[-B,B]$, so that $u_{k+1}=u_k+\epsilon$ and $[-B,B]=\bigcup_{k=1}^{M_u}[u_k-\epsilon/2,u_k+\epsilon/2]$.
Consider the following finite-dimensional, continuous optimization problem:
\begin{align}
q_1^*,\cdots,q_{M_u}^* &= \arg\max_{q_1,\cdots,q_{M_u}\in[\underline p,\overline p]} \sum_{k=1}^{M_u} \gamma_k r_{u_k}(q_k)\label{eq:proof-optimal-pi-1}\\
s.t.&\;\;\;\; \big|q_{k+1}^*-q_k^*\big|\leq \delta_0\epsilon, \;\; k=1,2,\cdots,M_u-1.
\end{align}
Given the optimal solution $q_1^*,\cdots,q_{M_u}^*$, a contextual policy $\pi_\epsilon^\dagger(x)=\varpi_\epsilon^\dagger(x^\top\theta_0)$ is constructed as
a piecewise linear interpolation:
$$
\varpi_\epsilon^\dagger(u) = q_k^* + \frac{u-u_k}{u_{k+1}-u_k}q_{k+1}^* \;\;\;\;\;\;\text{if }u_k\leq u<u_{k+1},
$$
with the understanding that $u_0=-B$, $q_0^*=q_1^*$, $u_{M_u+1}=B$ and $q_{M_u=1}^*=q_{M_u}^*$.
It is easy to verify that $\pi_\epsilon^*$ satisfies $\delta_0$-UF. Furthermore, the optimal solution $j_1^*,\cdots,j_{M_u}^*$ feasible to Eq.~(\ref{eq:inf-opt-ver3}) can be easily converted to a feasible solution $q_1=p_{j_1^*}$, $q_2=p_{j_2^*}$, etc.~
so that the resulting policies are exactly the same. Therefore, Lemma \ref{prop:approx} impllies that
\begin{align}
\mathbb E_{P_{\mU}}[r_u(\varpi_\epsilon^\dagger(u))] \geq \mathbb E_{P_{\mU}}[r_u(\varpi_\epsilon^*(u))] \geq \mathbb E_{P_{\mU}}[r_u(\varpi^*(u))] - 4L_f\delta_0\epsilon.
\label{eq:proof-optimal-pi-1}
\end{align}

We next use Lagrangian multipliers to analyze the properties of $q_1^*,\cdots,q_{M_u}^*$. For $k\in[M_u]$, let $\lambda_k^+\geq 0$ be the Lagrangian multiplier associated with the constraint $q_{k+1}-q_k-\delta_0\epsilon\leq 0$
and $\lambda_k^-\geq 0$ be the Lagrangian multiplier associated with the constraint $-q_{k+1}+q_k-\delta_0\epsilon\leq 0$. Let $q^*=(q_1^*,\cdots,q_{M_u}^*)$ be the vectorized optimal price solution 
and $\lambda^+ = (\lambda_1^+,\cdots,\lambda_{M_u}^+)$, $\lambda^-=(\lambda_1^-,\cdots,\lambda_{M_u}^-)$ be the vectorized optimal Lagrangian multipliers. The Lagrangian function (when expressing the original problem
as a minimization problem) can then be written as
$$
\mL(q^*,\lambda^+,\lambda^-) = \sum_{k=1}^{M_u} -\gamma_k r_{u_k}(q_k^*) + \lambda_k^+(q_{k+1}^*-q_k^*-\delta_0\epsilon)+ \lambda_k^-(q_k^*-q_{k+1}^*-\delta_0\epsilon).
$$

For every $k\in[M_u]$, define 
$$
\Delta_k := \lambda_k^+ - \lambda_k^-.
$$
Taking the partial derivative of $\mL$ with respect to $q_k^*$ for each $k\in[M_u]$ and using the first-order KKT condition \footnote{Strong duality holds here because 
all constraints are affine in the primal variables.}, it holds that
\begin{align}
\partial_{q_k}\mL(q^*,\lambda^+,\lambda^-) &= -\gamma_kr_{u_k}'(q_k) - (\lambda_k^+-\lambda_k^-) + \vct 1\{k>1\}(\lambda_{k-1}^+-\lambda_{k-1}^-)\nonumber\\
&= -\gamma_k r_{u_k}'(q_k^*) - \Delta_k+ \vct 1\{k>1\}\Delta_{k-1} = 0.\label{eq:proof-optimal-pi-2}
\end{align}
Subsequently, we have the recursion that
\begin{equation}
\Delta_{k} = \vct 1\{k>1\}\Delta_{k-1} - \gamma_k r_{u_k}'(q_k^*), \;\;\;\;\;\; k=1,2,\cdots,M_u.
\label{eq:proof-optimal-pi-3}
\end{equation}
Additionally, by complementary slackness (the constraints associated with $\lambda_k^+,\lambda_k^-$ cannot be binding simultaneously if $\delta_0\epsilon >0$,
and therefore at least one of $\lambda_k^+,\lambda_k^-$ must be zero) and the fact that $\lambda^+,\lambda^-\geq 0$, the following facts are true for every $k$:
\begin{align}
\Delta_k>0 &\Longrightarrow \lambda_k^+>0, \lambda_k^-=0\Longrightarrow q_{k+1}^*=q_k^*+\delta_0\epsilon;\label{eq:proof-optimal-pi-41}\\
\Delta_k<0 &\Longrightarrow \lambda_k^->0, \lambda_k^+=0\Longrightarrow q_{k+1}^*=q_k^*-\delta_0\epsilon;\label{eq:proof-optimal-pi-42}\\
\Delta_k=0& \Longrightarrow \lambda_k^+=\lambda_k^-=0\Longrightarrow q_k^*-\delta_0\epsilon < q_{k+1}^* < q_k^*+\delta_0\epsilon.\label{eq:proof-optimal-pi-43}
\end{align}

We are now ready to establish several properties of the optimal primal and dual solution $q^*,\lambda^+,\lambda^-$ that would eventually prove Theorem \ref{thm:optimal-pi}.
Without loss of generality, we shall assume that $P_{\mU}$ is supported on $[-B,B]$. If $P_{\mU}$ is instead supported on a closed sub-interval of $[-B,B]$ (Assumption \ref{asmp:stochasticity}), then the rest of the proof remains valid by simply restricting the utility valuation and discretization grid to the support of $P_{\mU}$.

{\textbf{Property 1:}} \underline{$\Delta_k\geq 0$ for all $k<M_u$. } Proof: assume by way of contradiction that there exists $k<M_u$ such that $\Delta_k<0$.
Let $k$ be the smallest integer such that $\Delta_k<0$. Because $\Delta_{k-1}\geq 0$ (with the understanding that $\Delta_0=0$), Eq.~(\ref{eq:proof-optimal-pi-3}) 
implies that $r_{u_k}'(q_k^*)>0$, which means $q_k^*< p^*(u_k)$ thanks to the uni-modality of $r_{u_k}(\cdot)$.
On the other hand, Eq.~(\ref{eq:proof-optimal-pi-42}) shows that $q_{k+1}^*=q_k^*-\delta_0\epsilon < q_k^*$, which combined with the fact that $p^*(u_{k+1})>p^*(u_k)$
(Lemma \ref{lem:pstar-mono}) shows $r_{u_{k+1}}'(q_{k+1}^*) > 0$, which combined with $\Delta_k<0$ and Eq.~(\ref{eq:proof-optimal-pi-3}) yields $\Delta_{k+1}<0$.
Continuing this argument until $k=M_u$, we have that $\Delta_{\ell}<0$ for all $\ell=k,k+1,\cdots,M_u$, and therefore $q_k^*>q_{k-1}^*>\cdots>q_{M_u}^*$.
Note that, on the other hand, $q_k^*<p^*(u_k)<p^*(u_{k+1})<\cdots p^*(u_{M_n})$. This means that if we set $q_{k+1}^*,\cdots,q_{M_u}^*$ to be $q_k^*$ we obtain
a solution that has strictly larger objective while remaining feasible to Eq.~(\ref{eq:proof-optimal-pi-1}), which contradicts the optimality of $\{q_k^*\}_{k=1}^{M_u}$. $\blacksquare$

{\textbf{Property 2:}} \underline{if $\Delta_{k-1}>0$, $\Delta_k=0$ for some $k<M_u-1$ then $q_{k+1}^*\geq p^*(u_k)>q_{k}^*$.}
Proof: because $\Delta_{k-1}<0$ and $\Delta_k=0$, Eq.~(\ref{eq:proof-optimal-pi-3}) implies that $r_{u_k}'(q_k^*)<0$, meaning that $q_k^*<p^*(u_k)$ thanks to the uni-modality of $r_{u_k}'$.
Assume by way of contradiction that $q_{k+1}^*<p^*(u_k)$. Because $p^*(u_{k+1})>p^*(u_k)$ thanks to Lemma \ref{lem:pstar-mono}, this assumption means that $r_{u_{k+1}}'(q_{k+1}^*)>0$
and therefore from Eq.~(\ref{eq:proof-optimal-pi-3}) one has that $\Delta_{k+1}=\Delta_k-\gamma_{k+1}r_{u_{k+1}}'(q_{k+1}^*) = -\gamma_{k+1}r_{u_{k+1}}'(q_{k+1}^*)<0$,
contradicting property 1. $\blacksquare$

Properties 1 and 2 immediately imply that $\varpi^*$ is a monotonically non-decreasing function in $u$, because in both cases of $\Delta_k>0$ and $\Delta_k=0$ the $q_k^*$ prices will not drop,
and by taking $\epsilon\to 0^+$ we obtain the monotonicity of the entire curve $\varpi^*$.

{\textbf{Property 3:}} \underline{if $\Delta_k=\Delta_{k-1}=0$ then $q_k^*=p^*(u_k)$.} Proof: it is immediate from Eq.~(\ref{eq:proof-optimal-pi-3}) that $r_{u_k}'(q_k^*)=0$ and therefore $q_k^*=p^*(u_k)$. $\blacksquare$

Property 3 can be used to establish the second property for the general optimal $\varpi^*$ solution. Because $\varpi^*$ is $\delta_0$-Lipschitz continuous (the utility fairness constraint),
Radamacher's theorem shows that $\varpi^*$ is differentiable almost everywhere on $[-B,B]$. Consider arbitrary $u\in(-B,B)$ such that $\varpi^*$ is differentiable at $u$ and $\partial_u\varpi^*(u)<\delta_0$.
By definition of differentials, there exist $\zeta,\beta\in(0,\delta_0)$ such that for every $u'\in[u-\zeta,u+\zeta]$, $|\varpi^*(u')-\varpi^*(u)|\leq (\delta_0-\beta)|u'-u|$.
Because $\pi^*$ is unique and $\lim_{\epsilon\to 0^+}\mathbb E_{P_\mU}[r_u(\varpi_\epsilon^\dagger(u))] = \mathbb E_{P_{\mU}}[r_u(\varpi^*(u))]$, Lemma \ref{lem:approx-uniform}
shows that $\lim_{\epsilon\to 0^+}\|\varpi_\epsilon^\dagger-\varpi^*\|_{\infty}=0$, meaning that for any small $\eta\in(0,\zeta/5]$, there exists $\epsilon_0>0$ such that for any $\epsilon\leq\epsilon_0$ and $u\in[-B,B]$,
$|\varpi_\epsilon^\dagger(u)-\varpi^*(u)|\leq\eta\beta$. For such a $\varpi_\epsilon^\dagger$, there must exist a discretized grid point $u_k\in[u-4\eta,u+4\eta]$ such that $\Delta_k=0$: 
otherwise, with $\Delta_k>0$ for all grid points in $[u-4\eta,u+4\eta]$ (the case of $\Delta_k<0$ is ruled out in property 1), the slope of $\varpi_\epsilon^\dagger$ is fixed locally at $\delta_0$ in $[u-4\eta,u+4\eta]$
and therefore $\varpi_\epsilon^\dagger(u-4\eta) = \varpi_\epsilon^\dagger(u)-4\delta_0\eta\leq \varpi_\epsilon^*(u)+\eta\beta-4\delta_0\eta$,
while on the other hand $\varpi_\epsilon^\dagger(u-4\eta)\geq \varpi^*(u-4\eta)-\eta\beta\geq \varpi^*(u)-4(\delta_0-\beta)\eta-\eta\beta = \varpi^*(u)+3\eta\beta-4\delta_0\eta$, which is a contradiction.
The existence of $u_k\in[u-4\eta,u+4\eta]$ such that $\Delta_k=0$ implies, by invoking Property 3, that $\varpi_\epsilon^\dagger(u_k+\epsilon)\geq p^*(u_k)>\varpi_\epsilon^\dagger(u_k)$.
By continuity of $\varpi_\epsilon^\dagger$ and picking $\epsilon\leq\eta$, there exists $\tilde u\in[u-5\eta,u+5\eta]$ such that $\varpi_\epsilon^\dagger(\tilde u)=p^*(u_k)$.
Now consider $\{\eta_n\}_{n\in\mathbb N}$ converging to zero and let $\{\epsilon_n\}_{n\in\mathbb N}$ also converging to zero such that $\|\varpi_{\epsilon_n}^\dagger-\varpi^*\|_{\infty}\leq\eta_n\beta$ for all $n$.
The above argument shows that there exist two sequences $\{v_n,\tilde v_n\}_{n\in\mathbb N}$ both converging to $u$ such that $\varpi_{\epsilon_n}^\dagger(v_n)=p^*(\tilde v_n)$,
and $|\varpi^*(v_n)-p^*(\tilde v_n)|\leq\epsilon_n$.
Because $\varpi^*$ is $\delta_0$-Lipschitz continuous, $\lim_{n\to\infty}\varpi^*(v_n)=\varpi^*(\lim_{n\to\infty}v_n)=\varpi^*(u)$.
By Lemma \ref{lem:pstar-mono}, $p^*(\cdot)$ is continuous and therefore $\lim_{n\to\infty}p^*(\tilde v_n)=p^*(\lim_{n\to\infty}\tilde v_n)=p^*(u)$.
Noting also that $\lim_{n\to\infty}\epsilon_n=0$, we have that $\lim_{n\to\infty}|\varpi^*(v_n)-p^*(\tilde v_n)|=0$, which implies $\varpi^*(u)=p^*(u)$.
This proves the second property in Theorem \ref{thm:optimal-pi} for general $\delta_0$ values.

In the remainder of this proof, we focus on the case when $\delta_0<\sigma_u/M_r$. 

{\textbf{Property 4}.} \underline{if $\delta_0<\sigma_u/M_r$ then $\Delta_k>0$ for all $k<M_u-1$.} Proof: assume by way of contradiction that $\Delta_k=0$ for some $k<M_u-1$ (the case of $\Delta_k<0$ has already been ruled out in property 1). Let $k$ be the smallest integer such that this happens. Because $\Delta_{k-1}\geq 0$ (with the understanding that $\Delta_0=0$), Eq.~(\ref{eq:proof-optimal-pi-3})
implies that $r_{u_k}'(q_k^*)\geq 0$, which means that $q_k^*\leq p^*(u_k)$ because $r_{u_k}(\cdot)$ is uni-modal.
Since $q^*$ is primal feasible, this means that $q_{k+1}^*\leq q_k^*+\delta_0\epsilon\leq p^*(u_k)+\delta_0\epsilon$. On the other hand, Lemma \ref{lem:pstar-mono} asserts that
$p^*(u_{k+1})\geq \frac{\sigma_r}{M_r}(u_{k+1}-u_k) = \frac{\sigma_r}{M_r}\epsilon$. Because $\delta_0<\sigma_r/M_r$, it holds that $q_{k+1}^*<p^*(u_{k+1})$ and therefore $r_{u_{k+1}}'(q_{k+1}^*)>0$.
Applying Eq.~(\ref{eq:proof-optimal-pi-3}) again, we obtain $\Delta_{k+1}=\Delta_k-\gamma_{u_{k+1}}r_{u_{k+1}}'(q_{k+1}^*)=-\gamma_{u_{k+1}}r_{u_{k+1}}'(q_{k+1}^*)<0$, which contradicts property 1. $\blacksquare$

Property 4 shows that with the additional condition of $\delta_0<\sigma_u/M_r$, we must have $\Delta_k>0$ for all $k<M_u-1$, implying that $\varpi^*$ must have the slope fixed to $\delta_0$ throughout. 
Finally, this structure can be extended to $\delta_0=\sigma_u/M_r$ by considering the sequence of optimal policies with the desirable linear structure for $\{\delta_{0n}\}_{n\in\mathbb N}$, $\lim_{n\to\infty}\delta_{0n}=\sigma_u/M_r$.
{By Lemma \ref{lem:optimal-obj-continuity}, the expected revenues of these optimal policies converge to the expected revenue of the optimal policy for $\delta_0=\sigma_u/M_r$}, 
which then implies uniform convergence of the policies themselves when the optimal policy is unique, thanks to Lemma \ref{lem:approx-uniform}.

\section{Proof of Theorem \ref{thm:cost-uf}}
In the purely linear demand case the optimal price (without fairness constraints) of a given baseline utility value $u$ is $p^*(u)=u/\alpha_0$, which has a slope of $1/\alpha_0$.
Therefore, if $\delta_0\geq 1/\alpha_0$, $R(\delta_0)$ would be exactly equal to $R(+\infty)$ by applying policy $p^*$ and therefore $\rho(\delta_0)=1$ for all $\delta_0\geq 1/\alpha_0$.

We next focus on the case of $\delta_0<1/\alpha_0$. Theorem \ref{thm:optimal-pi} shows that in this case, the optimal $\delta_0$-UF policy $\varpi^*$ takes the form of
$$
\varpi^*(u)=\pi_0 + \delta_0 u,
$$
with the only parameter that is free to vary being $\pi_0=\varpi^*(0)$. Let
\begin{align*}
\phi(\pi_0) &:= \mathbb E_{P_{\mU}}[r_u(\pi_0+\delta_0 u)],
\end{align*}
where $r_u(p)=p(u-\alpha_0 p/2)$ and $r_u'(p)=u-\alpha_0p/2+p(-\alpha_0/2) = u-\alpha_0 p$.
Taking derivative of $\phi$ with respect to $\pi_0$, we have that
\begin{align}
\phi'(\pi_0) &= \mathbb E_{P_{\mU}}[r_u'(\pi_0+\delta_0 u)] = \mathbb E[u-\alpha_0(\pi_0+\delta_0 u)] = -\alpha_0\pi_0 + (1-\alpha_0\delta_0)\mu_u.
\label{eq:proof-cost-uf-1}
\end{align}
Equating $\phi'(\pi_0)=0$ we have that $\pi_0=\alpha_0^{-1}(1-\alpha_0\delta_0)\mu_u$. Consequently,
\begin{align}
R(\delta_0) &= \mathbb E_{P_{\mU}}[r_u(\alpha_0^{-1}(1-\alpha_0\delta_0)\mu_u+ \delta_0 u)]\nonumber\\
& = \mathbb E_{P_{\mU}}\left[-\frac{\alpha_0}{2}\left(\frac{1-\alpha_0\delta_0}{\alpha_0}\mu_u+\delta_0 u\right)^2 +u\left(\frac{1-\alpha_0\delta_0}{\alpha_0}\mu_u+\delta_0 u\right) \right]\nonumber\\
&= -\frac{(1-\alpha_0\delta_0)^2}{2\alpha_0}\mu_u^2 -\delta_0(1-\alpha_0\delta_0)\mu_u^2 - \frac{\alpha_0\delta_0^2}{2}\nu_u^2 +\frac{1-\alpha_0\delta_0}{\alpha_0}\mu_u^2+\delta_0\nu_u^2\nonumber\\
&= \frac{(1-\alpha_0\delta_0)^2}{2\alpha_0}\mu_u^2 + \delta_0\left(1-\frac{\alpha_0\delta_0}{2}\right)\nu_u^2.\label{eq:proof-cost-uf-2}
\end{align}
On the other hand, when $\delta_0=1/\alpha_0$, the first term in Eq.~(\ref{eq:proof-cost-uf-2}) is zero and therefore
\begin{equation}
R(1/\alpha_0) = \frac{1}{2\alpha_0}\nu_u^2.\label{eq:proof-cost-uf-3}
\end{equation}
Note that for all $\delta_0\geq1/\alpha_0$, the optimal expected revenue is the same because the utility fairness constraint is mute here, or more specifically $R(1/\alpha_0)=R(+\infty)$.
Taking the ratio betwen Eqs.~(\ref{eq:proof-cost-uf-2},\ref{eq:proof-cost-uf-3}) we complete the proof of Theorem \ref{thm:cost-uf}.

\section{Proof of Theorem \ref{thm:lower-bound}}

Fix $\delta_0\in(0,1/2]$.
Set $d=1$, $\mX=[1/4, 3/4]$, $[\underline p,\overline p] = [0, 3\delta_0/4]$ and let $P_{\mX}$ be the uniform distribution on $\mX$.
Let $\Delta_\theta,\Delta_\alpha\in[0,\delta_0^2/10]$ be positive parameters to be specified later. Consider two hypothesis $H_0,H_1$ that differ only in $\theta$ and $\alpha$:
\begin{align*}
H_0: & \;\;\;\;\;\; \theta=1,\;\;\alpha=\delta_0^{-1};\\
H_1:& \;\;\;\;\;\; \theta=1-\Delta_\theta,\;\;\alpha=\delta_0^{-1}-\Delta_\alpha.
\end{align*}
The demand model is $\mathbb E[y|x,p]=x\theta-0.5\alpha p$, with $y\in\{0,1\}$ being Bernoulli random variables.
The expected revenue is $r_x(p)=p(x\theta-0.5\alpha p)$. It is easy to verify that, under both problem instances, $\mathbb E[y|x,p]\in[0,1]$ for all $x\in\mX$, $p\in[\underline p,\overline p]$.
Furthermore, the condition $\delta_0\leq \alpha_0^{-1}$ is true for both problem instances.

\subsection{Optimal policy and deviation analysis}

In this section we derive the optimal policy under problem instances parameterized by $\Delta_\theta,\Delta_\alpha$, and also analyze the property of sub-optimal policies that deviate from the optimal policy.
Our first lemma gives the properties of the optimal policy.
\begin{lemma}
Let $\mF_{\theta,\alpha}(\delta_0)$ be the set of all $\delta_0$-UF policies under problem instance $\theta,\alpha$.
Let $\pi^*\in\mF_{\theta,\alpha}(\delta_0)$ be the policy that maximizes the expected revenue under problem instance $\theta,\alpha$, and suppose that $\delta_0\leq 1/\alpha$.
For every $x\in\mX$, let $p^*(x)=\arg\max_{p\in[\underline p,\overline p]}r_x(p)$. Then for every $x\in\mX$,
\begin{enumerate}
\item $p^*(x) = \frac{\theta}{\alpha}$;
\item $\pi^*(x)=\frac{\theta(1-\alpha\delta_0)}{2\alpha} + \delta_0x\theta$;
\item $\mathbb E[r_x(\pi^*(x))] = \frac{\theta^2}{8}(\frac{1}{\alpha}+\frac{\delta_0}{6}+\frac{\delta_0^2}{12})$.
\end{enumerate}
\label{lem:lb-optimal}
\end{lemma}
\begin{proof}{Proof of Lemma \ref{lem:lb-optimal}.}
For the first property, taking $r'_x(p)=-\alpha p+x\theta = 0$ we obtain $p^*(x)=x\theta/\alpha$.
For the second property, note that $\pi^*$ must take the form of $\pi^*(x)=\pi_0^*+\delta_0x\theta$ for some $\pi_0^*\in\mathbb R$, thanks to Theorem \ref{thm:optimal-pi}.
Taking the derivative of the expected revenue with respect to $\pi_0$ and equating it to zero, we obtain
\begin{align}
\partial_{\pi}\mathbb E[r_x(\pi_0^*+\delta_0x\theta)] &= \mathbb E[r'_x(\pi_0^*+\delta_0 x\theta)] = \mathbb E[-\alpha(\pi_0^*+\delta_0 x\theta)+x\theta]=0,
\end{align}
which implies, together with the fact that $\mathbb E[x]=1/2$, 
$$
\pi_0^* = \frac{(1-\alpha\delta_0)\mathbb E[x\theta]}{\alpha} = \frac{(1-\alpha\delta_0)\theta}{2\alpha},
$$
proving the second property. For the third property, simply plugging the form of $\pi^*$ into the expected revenue evaluation we obtain
\begin{align*}
\mathbb E[r_x(\pi^*(x))] &= \mathbb E[(\pi_0^*+\delta_0 x\theta)(x\theta-0.5\alpha(\pi_0^*+\delta_0x\theta))]\\
&= -0.5\alpha(\pi_0^*)^2 + (1-\delta_0\alpha)\pi_0\mathbb E[x\theta] + \delta_0(1+0.5\alpha\delta_0)\mathbb E[x^2\theta^2]\\
&= -\frac{\alpha(\pi_0^*)^2}{2} + \frac{(1-\delta_0\alpha)\pi_0^*\theta}{2} + \frac{13}{48}{\delta_0(1+0.5\alpha\delta_0)\theta^2}\\
&= \frac{\theta^2}{8}\left(\frac{1}{\alpha} + \frac{\delta_0}{6} + \frac{\delta_0^2}{12}\right),
\end{align*}
which is to be proved. $\square$
\end{proof}

To simplify our notations, we shall use $\mF_0(\delta_0)$ and $\mF_1(\delta_0)$ to denote the class of $\delta_0$-UF policies under problem instances $H_0$ and $H_1$, respectively.
We also use $\mathbb E_0$ and $\mathbb E_1$ for the laws under $H_0$ and $H_1$.
For $x\in\mX$, $b\in\{0,1\}$ and $p\in[\underline p,\overline p]$, write
$$
r_{bx}(p) := r_{x^\top\theta}(p)=p(x^\top\theta-0.5\alpha p), \;\;\;\;\;\theta,\alpha\text{ associated with $H_b$.}
$$
Our next lemma shows that, any policy that is $\delta_0$-UF with respect to $H_1$ is going to incur a gap of $\Omega(\Delta_\theta)$ under $H_0$, compared with the optimal policy that
is $\delta_0$-UF with respect to $H_0$.
\begin{lemma}
For any $\pi\in\mF_1(\delta_0)$, it holds that
$$
\max_{\pi'\in\mF_0(\delta_0)}\mathbb E_0[r_{0x}(\pi'(x))] - \mathbb E_0[r_{0x}(\pi(x))] \geq \Delta_\theta/32.
$$
\label{lem:lb-h0-h1}
\end{lemma}
\begin{proof}{Proof of Lemma \ref{lem:lb-h0-h1}.}
Because the constructed problem instances are one-dimensional and $\theta=1-\Delta_\theta$ in $H_1$ for some $\Delta_\theta>0$,
it is easy to verify that $\mF_1(\delta_0)\subseteq\mF_0(\delta_0-\Delta_\theta)$.
Consequently, the third property of Lemma \ref{lem:lb-optimal} implies that, the left-hand side of the inequality in Lemma \ref{lem:lb-h0-h1} is lower bounded by
\begin{align*}
\frac{\theta^2}{8}&\left(\frac{1}{\alpha} + \frac{\delta_0}{6} + \frac{\delta_0^2}{12}\right) - \frac{\theta^2}{8}\left(\frac{1}{\alpha} + \frac{\delta_0-\Delta_\theta}{6} + \frac{(\delta_0-\Delta_\theta)^2}{12}\right)\nonumber\\
&\geq \frac{\theta^2}{8}\left(\frac{\Delta_\theta}{3}-\frac{\Delta_\theta^2}{12}\right) \geq \frac{\Delta_\theta}{32},
\end{align*}
which is to be proved. $\square$
\end{proof}

Our next lemma shows that, under problem instance $H_0$, if a $\delta_0$-UF contextual pricing policy $\pi\in\mF_0(\delta_0)$
deviates significantly from the optimal $\delta_0$-UF policy $\pi^*$, then it must also have a significantly smaller expected revenue.
\begin{lemma}
Let $\pi^*=\arg\max_{\pi\in\mF_0(\delta_0)}\mathbb E_0[r_{0x}(\pi(x))]$ and $\pi$ be arbitrary. Then
$$
\mathbb E_0[r_{0x}(\pi^*(x))] - \mathbb E_0[r_{0x}(\pi(x))] = \int_{1/4}^{3/4} \big|\pi(x)-\pi^*(x)\big|^2\ud x.
$$
\label{lem:lb-suboptimal}
\end{lemma}
\begin{proof}{Proof of Lemma \ref{lem:lb-suboptimal}.}
Immediate from the fact that, under $H_0$, $\pi^*(x)=p^*(x)$ and $r_{0x}(p^*(x))-r_{0x}(p)=(p-p^*(x))/2$ for all $x\in\mX$. $\square$
\end{proof}

\subsection{Bounding KL-Divergence}

For $b\in\{0,1\}$, let $\pi_b^*\in\arg\min_{\pi\in\mF_b(\delta_0)}\mathbb E[r_{bx}(\pi(x))]$ be the optimal $\delta_0$-UF pricing policy under problem instance $H_b$.
Also define the following events:
\begin{align*}
\text{Event $\mE_b$}:&\;\;\;\; \pi_1,\cdots,\pi_T\in\mF_b(\delta_0); \\
\text{Event $\mW_b$}:&\;\;\;\; \sum_{t=1}^T\int_{1/4}^{3/4} r_{bx}(\pi_b^*(x))-r_{bx}(\pi_t(x)) \leq 2\underline C\times T^{2/3}.
\end{align*}
Intuitively, event $\mE_b$ captures the good event that implemented policies are $\delta_0$-UF under problem instance $H_b$,
and event $\mW_b$ captures the good event that the cumulative regret of implemented policies over $T$ time periods are small under problem instance $H_b$.

Our next lemma shows that, when $\Delta_\theta$ is not too small, the events defined cannot co-exist, meaning that a sequence of policies $\pi_1,\cdots,\pi_T$
cannot be $\delta_0$-UF with respect to $H_1$ while still achieving small regret under $H_0$ simultaneously.
\begin{lemma}
Suppose $\Delta_\theta > 128\underline C\times T^{-2/3}$. Then $\mE_1\cap\mW_0=\emptyset$.
\label{lem:lb-exclusive}
\end{lemma}
\begin{proof}{Proof of Lemma \ref{lem:lb-exclusive}.}
By Lemma \ref{lem:lb-h0-h1}, $\mE_1$ implies that, under $H_0$, 
\begin{align}
\sum_{t=1}^T \int_{1/4}^{3/4} [r_{0x}(\pi_0^*(x))-r_{0x}(\pi_t(x))]\ud x \geq \frac{\Delta_\theta T}{64}.
\label{eq:proof-lb-exclusive-1}
\end{align}
On the other hand, the definition of $\mW_0$ implies that, under $H_0$,
\begin{align}
\sum_{t=1}^T \int_{1/4}^{3/4} [r_{0x}(\pi_0^*(x))-r_{0x}(\pi_t(x))]\ud x \leq 2\underline C\times T^{2/3}.
\label{eq:proof-lb-exclusive-2}
\end{align}
With $\Delta_\theta>128\underline C\times T^{-2/3}$, Eqs.~(\ref{eq:proof-lb-exclusive-1},\ref{eq:proof-lb-exclusive-2}) are contradictory to each other and therefore the lemma is proved. $\square$
\end{proof}

The following lemma upper bounds the Kullback-Leibler (KL) divergence between the observables under $H_0$ and $H_1$, using the fact that $\mA$ is near-optimal under $H_0$.
\begin{lemma}
Suppose $\Delta_\alpha=2\Delta_\theta/\delta_0$.
For any algorithm $\mA$ that satisfies
\begin{equation}
\mathbb E_0^{\mA}\left[\sum_{t=1}^T r_{0x}(\pi_0^*(x))-r_{0x}(\pi_t(x))\right]\leq \underline C\times T^{2/3},
\label{eq:stat-lb-kl}
\end{equation}
it holds that
$$
\kl(P_0^{\mA}\|P_1^{\mA}) \leq \underline C\times \frac{44\Delta_\theta^2 T^{2/3}}{\delta_0^2} ,
$$
where $P_b^{\mA}$ is the law of $\{(x_t,p_t,y_t)\}_{t=1}^T$ under problem instance $H_b$ and algorithm $\mA$, for $b\in\{0,1\}$.
\label{lem:lb-kl}
\end{lemma}
\begin{proof}{Proof of Lemma \ref{lem:lb-kl}.}
By Lemma \ref{lem:lb-suboptimal}, Eq.~(\ref{eq:stat-lb-kl}) implies that
\begin{equation}
\sum_{t=1}^T\int_{1/4}^{3/4} \big|\pi_t(x)-\pi_0^*(x)\big|^2\ud x \leq 0.5\underline C\times T^{2/3}.
\label{eq:proof-lb-kl-1}
\end{equation}
Note that $\pi_0^*$ takes the closed-form of $\pi_0^*(x)=\delta_0x$. With $\Delta_\alpha=2\Delta_\theta/\delta_0$, it holds that 
\begin{equation}
\mathbb E_0[y|x,\pi_0^*(x)]=\mathbb E_1[y|x,\pi_0^*(x)], \;\;\;\;\;\forall x\in\mX,
\end{equation}
because $x-0.5\delta_0^{-1}\pi_0^*(x)=x(1-\Delta_\theta)-0.5(\delta_0^{-1}-\Delta_\alpha)\pi_0^*(x)$ for all $x$.
Note also that $\mathbb E_0[y|x,\pi_0(x^*)]\in [1/8, 3/8]$ for all $x$. Subsequently, by \citep[Lemma 3]{Chen2018b}, 
\begin{align}
\kl(P_0(\cdot|x,p)\|P_1(\cdot|x,p)) &\leq \Delta_\alpha^2(p-\pi_0^*(x))^2\left(8+\frac{8}{3}\right)\leq 11\Delta_\alpha^2(p-\pi_0^*(x))\nonumber\\
&= \frac{44\Delta_\theta^2}{\delta_0^2}\times (p-\pi_0^*(x))^2,\;\;\;\;\;\forall x\in\mX,\label{eq:proof-lb-kl-2}
\end{align}
where $P_b(\cdot|x,p)$ is the law of $y$ conditioned on $x,p$ under problem instance $H_b$, $b\in\{0,1\}$.
Combining Eqs.~(\ref{eq:proof-lb-kl-1},\ref{eq:proof-lb-kl-2}) and using the Markovian structure of the observations across $T$ time periods, we have that
\begin{align}
\kl(P_0^{\mA}\|P_1^{\mA}) &= \mathbb E_0^{\mA}\left[\sum_{t=1}^T 2\int_{1/4}^{3/4}\kl(P_0(\cdot|x,\pi_t(x))\|P_1(\cdot|x,\pi_t(x)))\ud x\right]\nonumber\\
&\leq \frac{44\Delta_\theta^2}{\delta_0^2}\mathbb E_0^{\mA}\left[\sum_{t=1}^T \int_{1/4}^{3/4}(\pi_t(x)-\pi_0^*(x))^2\ud x\right]\nonumber\\
&\leq \underline C\times \frac{44\Delta_\theta^2 T^{2/3}}{\delta_0^2}.
\end{align}
This proves Lemma \ref{lem:lb-kl}. $\square$
\end{proof}

\subsection{Completing the proof of Theorem \ref{thm:lower-bound}}

Instantiate $\Delta_\theta=130\underline C\times T^{-2/3}$ and $\Delta_\alpha=2\Delta_\theta/\delta_0^2$. For sufficiently large $T$, such $\Delta_\theta,\Delta_\alpha$
instantiated will satisfy all range conditions at the beginning of this proof.
Consider an algorithm $\mA$ such that $\min_{b\in\{0,1\}}\Pr_b^\mA[\mE_b]\geq 0.95$.
Assume by way of contradiction that 
\begin{equation}
\mathbb E_0^{\mA}\left[\sum_{t=1}^T r_{0x}(\pi_0^*(x))-r_{0x}(\pi_t(x))\right]\leq \underline C\times T^{2/3}.
\label{eq:lb-bwoc}
\end{equation}
If Eq.~(\ref{eq:lb-bwoc}) does not hold, then the conclusion of Theorem \ref{thm:lower-bound} automatically holds and the proof is done.

By Markov inequality, Eq.~(\ref{eq:lb-bwoc}) implies that $\Pr_0^{\mA}[\mW_0^c] \leq 1/2$, and therefore $\Pr_0^{\mA}[\mW_0]\geq 1/2$.
Because $\mW_0\cap\mE_1=\emptyset$ thanks to Lemma \ref{lem:lb-exclusive}, $\Pr_0^{\mA}[\mW_0]\geq 1/2$ implies that $\Pr_0^{\mA}[\mE_1]\leq 1/2$.
Now apply Lemma \ref{lem:lb-kl} and Pinsker's inequality; we obtain
\begin{align}
Pr_1^{\mA}[\mE_1] &\leq Pr_0^{\mA}[\mE_1] + \|P_0^{\mA}-P_1^{\mA}\|_{\mathrm{TV}} \leq \frac{1}{2} + \sqrt{\frac{\kl(P_0^{\mA}\|P_1^{\mA})}{2}}\nonumber\\
&\leq \frac{1}{2} + \sqrt{\underline C\times \frac{22\Delta_\theta^2 T^{2/3}}{\delta_0^2}} = \frac{1}{2} + \sqrt{\frac{371800\underline C^3}{\delta_0^2}}.\label{eq:proof-lb-final-1}
\end{align}
Set
$$
\underline C = \sqrt[3]{\frac{\delta_0^2}{371800\times 5}} = \delta_0^{2/3}/123,
$$
a strictly positive constant that only depends on $\delta_0$.
The right-hand side of Eq.~(\ref{eq:proof-lb-final-1}) is then upper bounded by $\frac{1}{2}+\frac{1}{\sqrt{5}} = 0.947<0.95$, which contradicts $\Pr_1^{\mA}[\mE_1]\geq 0.95$.
This completes the proof of Theorem \ref{thm:lower-bound}.

\section{Proof of Theorem \ref{thm:upper-bound}}

In this section we prove both fairness and regret claims in Theorem \ref{thm:upper-bound}.
Throughout this proof we assume that all assumptions \ref{asmp:boundedness}-\ref{asmp:likelihood} are valid, and will not explicitly invoke or state them when deriving results.

\subsection{Analysis of $\hat\theta$ and $\hat\alpha$}

We first establish a technical lemma shows that $\hat\theta$ and $\hat\alpha$ estimate $\theta_0$ and $\alpha_0$ in $\ell_2$ norm, with the estimation errors on the order of $\tilde O(1/\sqrt{T_0})=\tilde O(T^{-1/3})$.

\begin{lemma}
For $T$ being sufficiently large, with probability $1-\tilde O(T^{-2})$ it holds that
$$
\|\hat\theta-\theta_0\|_2\leq  \frac{4M_r}{\sigma_1}\sqrt{\frac{\ln(dT)}{T_0}}.
$$
\label{lem:pilot}
\end{lemma}
\begin{proof}{Proof of Lemma \ref{lem:pilot}.}
Let $p_L=\underline p$ and $p_U=\overline p$ be the two exploration prices.
Let $z=(x,-0.5p)$ be the extended context vector, and $\beta_0=(\theta_0,\alpha_0)$ be the extended model. Let $P_{\mZ}$ be the distribution of $z$ over the first $T_0$ time periods.
To simplify notations, for every $\beta\in\mathbb R^d$, define
$$
\varphi(\beta) := \sum_{t=1}^{T_1}\ln\mL(y_t|z_t,\beta).
$$
Because $\hat\beta$ is an unconstrained maximizer of $\varphi$, it holds that $\nabla\varphi(\hat\beta)=0$. Subsequently,
\begin{align}
-\langle\hat\beta-\beta_0,\nabla\varphi(\beta_0)\rangle &= -\int_0^1(\hat\beta-\beta_0)^\top\nabla^2\varphi(\beta_0+s(\hat\beta-\beta_0))(\hat\beta-\beta_0)\ud s\label{eq:proof-pilot-1}\\
&\geq -\int_0^{\min\{1,\rho_L/\|\hat\beta-\beta_0\|_2\}}(\hat\beta-\beta_0)^\top\nabla^2\varphi(\beta_0+s(\hat\beta-\beta_0))(\hat\beta-\beta_0)\ud s\label{eq:proof-pilot-2}\\
&\geq \min\{1,\rho_L/\|\hat\beta-\beta_0\|_2\}\times \sigma_L (\hat\beta-\beta_0)^\top \Lambda(\hat\beta-\beta_0),\label{eq:proof-pilot-3}
\end{align}
where 
\begin{equation}
\Lambda := \sum_{t=1}^{T_0}z_tz_t^\top.
\label{eq:proof-pilot-4}
\end{equation}
Here, Eq.~(\ref{eq:proof-pilot-1}) is a consequence of the fundamental theorem of calculus and the fact that $\nabla\varphi(\hat\beta)=0$;
Eq.~(\ref{eq:proof-pilot-2}) holds because $\varphi$ is concave thanks to the second property of Assumption \ref{asmp:likelihood}, and therefore $\nabla^2\varphi(\cdot)$ is negative semi-definite;
Eq.~(\ref{eq:proof-pilot-3}) holds by the third property of Assumption \ref{asmp:likelihood}.

Define $\Lambda_0 := \mathbb E_{P_{\mZ}}[zz^\top]$.
We then have that $\Lambda_0\succeq \cov(P_{\mZ})\succeq \min\{\sigma_x,0.25(\overline p-\underline p)^2\} I_{d+1} =: \sigma_0 I_{d+1}$.
On the other hand, because $z$ is bounded almost surely, using the matrix Hoeffding's inequality we have with probability $1-\tilde O(T^{-2})$ that
$\Lambda =\sum_{t=1}^T z_tz_t^\top \succeq 0.5\sigma_0 T_0 I_{d+1}$,
provided that $T$ is sufficiently large.
This combined together with Eqs.~(\ref{eq:proof-pilot-3},\ref{eq:proof-pilot-4}) yields that
\begin{align}
-\langle\hat\beta-\beta_0,\nabla\varphi(\beta_0)\rangle &\geq \sigma_1\times \min\left\{1, \frac{\rho_L}{\|\hat\beta-\beta_0\|_2}\right\}\times T_0\|\hat\beta-\beta_0\|_2^2,
\label{eq:proof-pilot-5}
\end{align}
where $\sigma_1 :=0.5\sigma_0\sigma_r > 0$.

On the other hand, by Hoeffding's inequality and the fact that $\mathbb E[\nabla\varphi(\beta_0)] = 0$ (property 1 of Assumption \ref{asmp:likelihood}), it holds with probability $1-\tilde O(T^{-2})$ that
\begin{align*}
\|\nabla\varphi(\beta_0)\|_2 &\leq 4M_r\sqrt{T_0\ln(dT)}.
\end{align*}
Consequently,
\begin{align}
\big|\langle\hat\beta-\beta_0,\nabla\varphi(\beta_0)\rangle\big| \leq 4M_r\sqrt{T_0\ln(dT)}\times \|\hat\beta-\beta_0\|_2.
\label{eq:proof-pilot-6}
\end{align}
Combine Eqs.~(\ref{eq:proof-pilot-5},\ref{eq:proof-pilot-6}). We obtain
\begin{align}
\min\{\rho_L, \|\hat\beta-\beta_0\|_2\}\leq \frac{4M_r}{\sigma_1}\sqrt{\frac{\ln(dT)}{T_0}}.
\label{eq:proof-pilot-7}
\end{align}
With $\rho_L>0$ being a constant and $T_0$ being sufficiently large, the left-hand side of the above inequality reduces to $\|\hat\beta-\beta_0\|_2$.
This then proves Lemma \ref{lem:pilot}. $\square$
\end{proof}

With Lemma \ref{lem:pilot} and the value of $\kappa_1$, we have with high probability that $\|\hat\theta-\theta_0\|_2\leq\kappa_1/\sqrt{T_0}$.
This immediately proves the fairness claim in Theorem \ref{thm:upper-bound}, because for every $x,x'\in\mX$, 
\begin{align*}
\tilde\delta_0\big|(x-x')^\top\hat\theta\big| &\leq \tilde\delta_0(|(x-x')^\top\theta_0| + |(x-x')^\top(\hat\theta-\theta_0)|)\leq (\tilde\delta_0 + \diam(\mX)\|\hat\theta-\theta_0\|_2)|(x-x')^\top\theta_0|\\
&\leq (\tilde\delta_0+\kappa_1/\sqrt{T_0})|(x-x')^\top\theta_0| \leq \delta_0|(x-x')^\top\theta_0|.
\end{align*}

\subsection{Analysis of the MAB-UCB procedure}

Given $\hat\theta$, the estimate of the linear model $\theta_0$ from the price experimentation phase, the expected revenue of a policy with a ``starting price'' $\pi_0$ can be written as
$$
\Phi(\pi_0) := \mathbb E[r_x(\pi_0+\tilde\delta_0x^\top\hat\theta)],
$$
where $r_x(p)=pf(x^\top\theta_0-\alpha_0 p)$. Recall also the definition that $R(\delta_0)$ is the expected revenue of the optimal $\delta_0$-UF policy,
with full information of $\theta_0$ and $\alpha_0$. The following lemma describes several properties of $\Phi(\cdot)$.

\begin{lemma}
Let $k^* = \arg\max_{k\in[K]}\Phi(\pi_k)$. Then $\Phi(\pi_{k^*})\geq R(\delta_0)-6L_f\tilde B\kappa_1/\sqrt{T_0} -2L_f(\overline p-\underline p)/K$. 
\label{lem:Phi}
\end{lemma}
\begin{proof}{Proof of Lemma \ref{lem:Phi}.}
Let $\pi^*$ be the optimal $\delta_0$-UF policy that maximizes the expected revenue, which admits the form of $\pi^*(x)=\pi_0^*+\delta_0x^\top\theta_0$ thanks to Theorem \ref{thm:optimal-pi}.
Let $\pi_k$ be the discretized price such that $|\pi_k-\pi_0^*|\leq (\overline p-\underline p)/K$. 
Because $r_x(\cdot)$ is $2L_f$-Lipschitz continuous, we have that
\begin{align}
\Phi(\pi_{k^*}) &\geq \Phi(\pi_k) = \mathbb E[r_x(\trim_{[\underline p,\overline p]}(\pi_k+\tilde\delta_0x^\top\hat\theta))] \nonumber\\
&\geq R(\delta_0) - 2L_f\big(\big|\pi_k-\pi_0^*\big|+\big|\tilde\delta_0x^\top\hat\theta-\delta_0x^\top\theta_0\big|\big)\nonumber\\
&\geq R(\delta_0) - 2\delta_0 L_f\diam(\mX)\times \|\hat\theta-\theta_0\|_2 - 4L_f(\delta_0-\tilde\delta_0)\times \tilde B- 2L_f(\overline p-\underline p)/K\nonumber\\
&\geq R(\delta_0) - 6L_f\tilde B\kappa_1/\sqrt{T_0} -2L_f(\overline p-\underline p)/K,
\end{align}
which proves Lemma \ref{lem:Phi}. $\square$
\end{proof}

\subsection{Regret analysis}
	
	The regret cumulated in the first price experimentation phase is upper bounded by $T_0=T^{2/3}$.
	
	For the second MAB phase, note that with the parameter $\kappa_2$ chosen as $\kappa_2=4\sqrt{\ln T}$, the Hoeffding's inequality and a union bound
	over all $K=\lceil T^{1/3}\rceil$ arms and $T-T_0$ time periods imply that, with probability $1-\tilde O(T^{-2})$, 
	\begin{align}
	\left|\frac{r_k}{n_k}-\Phi(\pi_k)\right| \leq \frac{\kappa_2}{\sqrt{n_k}}, \;\;\;\;\;\;\forall k\in[K].
	\label{eq:proof-regret-ub-1}
	\end{align}
	Subsequently, using the standard upper-confidence bound analysis, it holds that
	\begin{align}
	\sum_{t=T_0+1}^T \Phi(\pi_{k^*}) - \Phi(\pi_{k_t}) \leq 2\kappa_2\sqrt{KT},
	\label{eq:proof-regret-ub-2}
	\end{align}
	where $k^*=\arg\max_{k\in[K]}\Phi(\pi_k)$. Eq.~(\ref{eq:proof-regret-ub-2}) together with Lemma \ref{lem:Phi} shows that the cumulative regret of the second phase
	is upper bounded by 
	\begin{align}
	2\kappa_2\sqrt{KT} + &T\times (6L_f\tilde B\kappa_1/\sqrt{T_0} +2L_f(\overline p-\underline p)/K)\nonumber\\
	&= 2\kappa_2 T^{2/3} + 6L_f\tilde B\kappa_1 T^{2/3}+2L_f(\overline p-\underline p)T^{2/3}\nonumber\\
	&\leq (6L_f\tilde B\kappa_1 + 4L_f\kappa_2) T^{2/3},
	\end{align}
	which is to be proved.

\end{document}